\documentclass{article}

\usepackage[preprint]{neurips_2026}

\usepackage[utf8]{inputenc} 
\usepackage[T1]{fontenc}    

\usepackage{url}            
\usepackage{booktabs}       
\usepackage{amsfonts}       
\usepackage{nicefrac}       
\usepackage{microtype}      
\usepackage{xcolor}         

\usepackage{graphicx}
\usepackage{xcolor}

\usepackage{amsmath}
\usepackage{amssymb}
\usepackage{amsthm}
\usepackage{mathtools}

\usepackage{threeparttable}
\usepackage{arydshln}
\usepackage{array}
\usepackage{multirow}
\usepackage{makecell}
\usepackage{longtable}
\usepackage{siunitx}
\usepackage{algorithm}
\usepackage{algpseudocode}
\usepackage{enumitem}

\usepackage{caption}
\usepackage{subcaption}
\usepackage{float}
\usepackage{wrapfig}
\usepackage{colortbl}
 \usepackage{tcolorbox}
\usepackage{listings}

\usepackage{hyperref}
 
\hypersetup{
    colorlinks=true,
    citecolor=teal,
    linkcolor=teal,
    urlcolor=teal
}

\usepackage{xcolor}
\definecolor{darkteal}{RGB}{0,102,102}
\definecolor{darkgreen}{RGB}{0,102,0}
\definecolor{saddlebrown}{RGB}{139,69,19}
\newcommand{\an}[1]{\textcolor{black}{#1}}

\usepackage{cleveref}
\newtheorem{theorem}{Theorem}
\newtheorem{lemma}{Lemma}

\newtheorem{definition}{Definition}

\DeclareMathOperator*{\argmax}{arg\,max}

\newcommand{\tomlistener}[1]{L_{\textrm{ToM}_{#1}}}

\newcommand{\boundspeaker}[1]{S_{\textrm{bps}_{#1}}}
\newcommand{\basespeaker}[1]{S_{\textrm{base}_{#1}}}
\newcommand{\utterance}{u}
\newcommand{\intention}{z}
\newcommand{\context}{c}
\newcommand{\craft}{\textsc{CRAFT}}

\definecolor{blockblue}{RGB}{37,99,235}
\definecolor{blockorange}{RGB}{255,123,0}
\definecolor{blockgreen}{RGB}{34,197,94}
\definecolor{blockyellow}{RGB}{245,197,24}
\definecolor{blockred}{RGB}{220,38,38}
 \tcbuselibrary{breakable, skins}

\newcommand{\bblue}[1]{\texttt{\textcolor{blockblue}{#1}}}
\newcommand{\borange}[1]{\texttt{\textcolor{blockorange}{#1}}}
\newcommand{\bgreen}[1]{\texttt{\textcolor{blockgreen}{#1}}}
\newcommand{\byellow}[1]{\texttt{\colorbox{gray!12}{\textcolor{blockyellow}{#1}}}}
\newcommand{\bred}[1]{\texttt{\textcolor{blockred}{#1}}}
\newcommand{\failbox}[1]{\colorbox{red!10}{\strut #1}}
\newcommand{\okbox}[1]{\colorbox{green!10}{\strut #1}}

\usepackage{lineno}

\title{\craft{}: Grounded Multi-Agent Coordination Under Partial Information}

\author{%
  Abhijnan Nath
  \And 
  Hannah VanderHoeven
    \And
  Nikhil Krishnaswamy
  \And \\ 
  Situated Grounding and Natural Language (SIGNAL) Lab\thanks{\url{https://www.signallab.ai}} \\
  Department of Computer Science, Colorado State University\\
  Fort Collins, CO 80523 USA \\
  \texttt{\{abhijnan.nath,nkrishna\}@colostate.edu} \\
}

\begin{document}

\maketitle

\begin{abstract}
  We introduce \craft{}, a multi-agent benchmark for evaluating pragmatic communication in large language models (LLMs) under strict partial information conditions. In this setting, multiple agents with complementary but incomplete views must coordinate through natural language to construct a single 3D structure whose final form is \textit{unknown to any individual agent}. We formalize this as a multi-sender Bounded Pragmatic Speaker problem and provide an evaluation and diagnostic framework that decomposes failures into \textit{spatial grounding}, \textit{mind modeling} and \textit{pragmatic communication} errors. Across 8 open-weight and 7 frontier models, we find that stronger reasoning ability does not reliably translate to better coordination, and improved individual communication does not guarantee successful collaboration. These results suggest that multi-agent coordination remains a fundamentally unsolved challenge for current LLMs, and \craft{} provides a novel framework to meaningfully advance evaluation in multi-agent collaboration. Source code and benchmark data: \url{https://github.com/csu-signal/CRAFT}
\end{abstract}


\vspace*{-2mm}
\section{Introduction}
\label{sec:intro}
\vspace*{-2mm}

\an{Large language models (LLMs) have evolved from single-turn 
assistants into components of multi-agent systems where models 
must coordinate toward shared goals through open-ended natural 
language~\citep{kazemitabaar2023novices, zhang2024comprehensive, 
cui2024survey, davidson2025collaborationgap}. Yet coordination 
between LLM agents remains fragile even in simple 
settings~\citep{lupu2021trajectory, agashe-etal-2025-llm, 
singh2025malmm, nath2025learning, riedl2025emergent, 
davidson2025collaborationgap, 
eisenstein2026mtpingevalevaluatingmultiturncollaboration}: 
models struggle with partner modeling and joint decision-making 
even when interacting with identical copies of themselves---and 
critically, even when trained with Reinforcement Learning from 
Human Feedback (RLHF)~\citep{stiennon2020learning}, which 
optimizes for human-preferred responses but does not explicitly 
train for the communicative demands of multi-agent coordination. 
This challenge compounds under \emph{partial observability}, 
where agents hold differing private information and must decide 
what to communicate based on others' knowledge states.}

We argue that the missing capability is \emph{pragmatic communication}---deciding what and how much to say and when to say it, based on other agents’ knowledge and needs~\citep{grice1975logic, frank2012predicting}. While frameworks such as Rational Speech Acts (RSA)~\citep{goodman2013knowledge} and the Bounded Pragmatic Speaker (BPS) model~\citep{nguyen2024languagemodelsboundedpragmatic} formalize this behavior, existing evaluations largely focus on single-agent reasoning or offline interpretation~\citep{jian2024llms, zhu2026distributedpartialinformationpuzzles}. This creates a key gap: LLMs may exhibit strong internal reasoning (formal competence) but fail to use it effectively in interaction (functional competence)~\citep{mahowald2024dissociatinglanguagethoughtlarge, davidson2025collaborationgap}. 

\vspace{-1mm}

\an{Consider the Distributed Partial Information Puzzle (DPIP;~\citealt{zhu2026distributedpartialinformationpuzzles}): 3 \textbf{Directors}, each holding a distinct partial view of a 3D structure, must instruct a \textbf{Builder} to reconstruct it through natural language alone---a multi-party coordination structure our work directly adopts for LLMs. Humans succeed in DPIP through sustained pragmatic repair across multiple turns\footnote{\Cref{tab:dialogue} in Appendix~\ref{app:failure_taxonomy} shows a snippet from a DPIP ``game'' where humans take $\sim$30s and 23 turns of dialogic exchanges involving pragmatic repair to resolve a \textit{single} action.}, producing precisely the complementary, non-redundant coordination that RSA and BPS theory predict–and that collective intelligence research identifies as necessary for group-level synergy~\citep{humphreys1997properties}.}

\an{Yet two critical questions remain unanswered. First, while~\citep{nguyen2024languagemodelsboundedpragmatic} suggests RLHF-trained LLMs approximate bounded pragmatic speakers in dyadic settings, it is unknown whether this extends to multi-agent coordination where each agent simultaneously holds distinct private information and must model multiple partners' knowledge states. Second, existing evaluations in pragmatic-collaborative tasks for LLMs rely on offline analysis~\citep{ma2025pragmatics} or dyadic setups~\citep{wu2024your, davidson2025collaborationgap}---leaving open whether LLM collectives exhibit the integrated pragmatic “synergy” humans achieve in DPIP, or instead fall into “collaboration gaps”~\citep{davidson2025collaborationgap} that emerge when coordination under partial observability breaks down.
CRAFT is designed to answer both: \textbf{does the pragmatic competence that BPS theory predicts and humans exhibit in collaborative tasks like DPIP transfer to LLM collectives operating under the same epistemic asymmetry, and if not, where and how does it break down in LLM-native multi-agent coordination?}}

\begin{figure*}[t]
  \centering
  \includegraphics[width=\linewidth,clip,trim={20 10 20 10}]{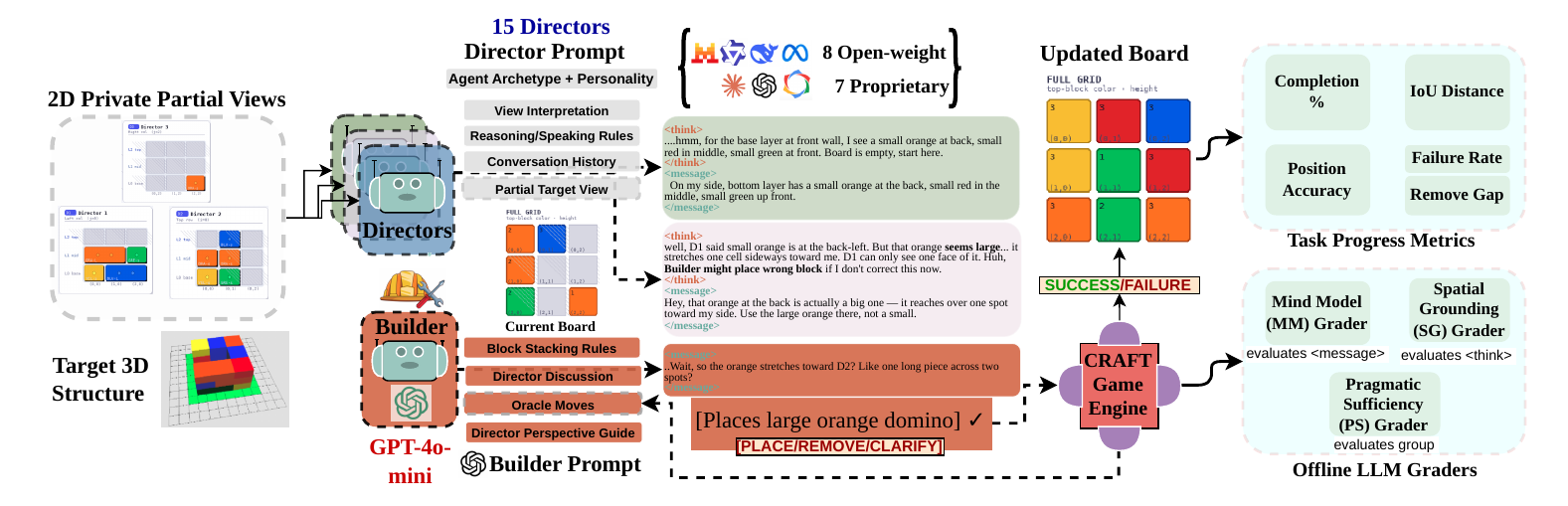}
\vspace*{-2mm}
\caption{\textbf{\craft{} overview.} A structure generator creates a target 3D object and 3 private 2D views for Directors, enforcing information asymmetry. Each turn, Directors produce instructions from their partial views, and a Builder executes them in the \craft{} environment, which logs task progress and evaluates communication for \textbf{spatial grounding}, \textbf{mind modeling}, and \textbf{pragmatic sufficiency}.
\vspace*{-2mm}
}

  \label{fig:craft_overview}
\end{figure*}

 \an{To answer this rigorously and at scale, we  introduce three original contributions: a \textit{procedural structure generator} with enforced information asymmetry; a \textit{physically-grounded game engine} with ground-truth move validation and an oracle-assisted Builder interface to disentangle Builder capability from Director communication, enabling proper credit assignment in multi-agent systems~\citep{huh2024multiagentreinforcementlearningcomprehensive}; and a \textit{3-judge evaluation framework} decomposing communication failures into spatial grounding of private reasoning, epistemic calibration of public messages to the group, and collective pragmatic sufficiency–each judge targeting a distinct failure mode predicted by our \textit{multi-sender BPS formalization} of directors as bounded pragmatic speakers~\citep{frank2012predicting,nguyen2024languagemodelsboundedpragmatic}. Our evaluation across \textbf{15} frontier and open-weight LLMs reveals that scale \textit{alone} does not confer reliable coordination advantages, that individual communication quality does not always lead to better collaboration, and that LLM collectives exhibit previously unknown behavioral patterns such as “correction spirals” under information asymmetry.  These findings expose a fundamental gap between individual communicative competence in LLM agents and collective task success. Fig.~\ref{fig:craft_overview} provides a high-level overview of \craft{}.}

\vspace{-0.2mm}

\vspace*{-3mm}
\section{Related Works}
\label{sec:rw}
\vspace*{-2mm}


\paragraph{Spatial reasoning and multi-agent coordination under 
partial observability}
Recent work has advanced spatial reasoning for LLMs across viewpoint 
consistency, 3D reasoning, and object localization~\citep{kamath2023s, 
ma20253dsrbench, du2024embspatial, 
li2025viewspatialbenchevaluatingmultiperspectivespatial, 
liu2025spatialreasoningmultimodallarge, zhang2025sphere, 
xu2026spatialbenchbenchmarkingmultimodallarge, yeh2026seeing}, with 
extensions to dialogue and navigation~\citep{bickmore2005social, 
gao2022dialfred, zhang2024tagmaptextbasedmap, hou2025driveagent, 
martorell2025textspacemappingabstract, 
zheng2025spatiotemporalllmreasoningenvironments}. In parallel, LLM 
coordination has been studied in symmetric games, zero-shot partner 
matching, and agentic pipelines~\citep{lupu2021trajectory, 
jiang2024fullydecentralizedcooperativemultiagent, 
agashe-etal-2025-llm, grotschla2025agentsnet, 
maslej2025artificialintelligenceindexreport, 
nath-etal-2025-frictional, nath2025collaborate, nath2025learning, 
singh2025malmm, chen2026wsmultiagentcommunicationtalks, 
hayler2026zeroshot}. However, spatial benchmarks are typically 
single-agent perception tasks, while coordination studies focus on 
symmetric or abstract settings with limited grounding in shared 
physical state~\citep{liu2025spatialreasoningmultimodallarge, 
Mohammadi_2025, tran2025multiagentcollaborationmechanismssurvey, 
chen2026wsmultiagentcommunicationtalks}. Real-world collaboration 
instead requires agents to operate under partial observability with 
complementary private evidence communicated over multiple 
turns~\citep{chen2024agentverse, 
liang2025llmhanabievaluatingmultiagentgameplays}, where coordination 
remains fragile even in simplified 
environments~\citep{tian2020joint, davidson2024evaluating, 
davidson2025collaborationgap, grotschla2025agentsnet, wu2025collabllm} 
and demands informative, concise communication~\citep{garrod2004conversation, 
davidson2025collaborationgap}. Motivated by distributed partial 
information tasks~\citep{zhu2026distributedpartialinformationpuzzles} 
and reference games~\citep{nunberg1978pragmatics, andreas2016reasoning, 
ma2025pragmatics}, \craft{} addresses this gap by evaluating pragmatic 
communication in grounded multi-agent settings with asymmetric spatial 
information, measuring both task completion and communication 
quality~\citep{amirkhani2022consensus}.

\vspace*{-2mm}
\paragraph{Pragmatic reasoning and bounded pragmatic speakers}
Pragmatic language use has been widely studied through Gricean accounts and Rational Speech Acts (RSA), where speakers select utterances by modeling a listener’s beliefs and informational needs~\citep{grice1975logic, frank2012predicting, goodman2013knowledge}. Prior work has explored pragmatics in controlled reference games, typically in two-agent settings~\citep{fried2018unified, khani2018planninginferencepragmaticssequential, nematzadeh2018evaluating, louis2020d, hu2023fine, ruis2023goldilocks, zhang2023coder, estienne-etal-2025-collaborative}, while recent work extends this perspective to LLMs via the Bounded Pragmatic Speaker (BPS) framework~\citep{nguyen2024languagemodelsboundedpragmatic}.Benchmarks such as DiPlomat show that LLMs struggle with context-sensitive implicature and situated interpretation~\citep{li2023diplomat}, and surveys highlight that \an{existing evaluations emphasize pragmatic 
\emph{understanding} over \emph{production}, yet 
LLMs perform substantially better as listeners than 
as speakers~\citep{ferreira2005psycholinguistics, 
meyer2016same, sieker2026hypocritical, 
krause-vossen-2024-gricean-maxims, 
park2024pragmaticcompetenceevaluationlarge}---a 
asymmetry that matters precisely because models may 
encode rich representations yet fail to deploy them 
in interaction~\citep{mahowald2024dissociatinglanguagethoughtlarge, 
wu2024your, davidson2025collaborationgap}. \craft{} 
targets this gap by separating the pragmatic 
“speaker” role from the “listener”
role, isolating production as the operative 
challenge in a multi-agent partially observable 
setting where joint action depends on it.}

\vspace{-2mm}

\vspace*{-2mm}
\section{\craft{} as a Multi-Sender Bounded Pragmatic Speaker Framework}
\vspace*{-2mm}

The Bounded Pragmatic Speaker (BPS) framework~\citep{nguyen2024languagemodelsboundedpragmatic}
provides a unified account of pragmatic language production in LLMs, in which a \textit{pragmatic speaker} selects utterances by jointly reasoning
about a base generative model and a Theory-of-Mind (ToM) listener that
evaluates how well an utterance communicates an intended meaning.



\begin{definition}[Bounded Pragmatic Speaker]
\label{def:bps}
Let $\basespeaker{}(\utterance \mid \intention^{\star}, \context)$ be a
\emph{base speaker} distribution over utterances $\utterance \in \mathcal{U}$,
conditioned on an intention $\intention^{\star} \in \mathcal{Z}$ and
context $\context$.
Let $\tomlistener{}(\intention^{\star} \mid \utterance, \context)$ be a
\emph{Theory-of-Mind (ToM) listener} that scores how faithfully
$\utterance$ communicates $\intention^{\star}$.
A \emph{Bounded Pragmatic Speaker} (BPS) selects utterances as:
\begin{align}
    \boundspeaker{}(\utterance \mid \intention^{\star}, \context)
    \;\propto\;
    \basespeaker{}(\utterance \mid \intention^{\star}, \context)
    \cdot
    \tomlistener{}(\intention^{\star} \mid \utterance, \context).
\label{eqn:bps}
\end{align}
\end{definition}

\an{\cite{nguyen2024languagemodelsboundedpragmatic} show that any language model $S_\theta$ can be
viewed as a BPS by setting both modules to $S_\theta$, and that
RLHF fine-tuning~\citep{stiennon2020learning, alignment_handbook2023} is equivalent to variational inference on a BPS whose
ToM listener is a learned reward function $R_\phi$.
We extend this framework to the \emph{multi-agent, grounded,
partial-information} setting of \craft{}.}

\vspace*{-3mm}
\subsection{Directors as Bounded Pragmatic Speakers}
\label{sec:directors_as_bps}
\vspace*{-2mm}

Director $D_i$ at turn $t$ has an \emph{intention} 
$z_{i,t}^{\star} = \Delta_i(s_t, \mathcal{T})$---the grounded 
gap between $D_i$'s target view $\mathcal{T}$ and the current 
world state $s_t$ (the observable board configuration at turn 
$t$)---and a \emph{context} $c_{i,t} = (o_{i,t},\; h_t,\; 
\{u_{j,t}\}_{j \neq i})$ comprising their private observation, 
the conversation history, and the current-turn utterances of 
the other Directors. Under the BPS framework, $D_i$'s policy is:
\begin{align}
    \pi_{D_i}(u_{i,t} \mid z_{i,t}^{\star}, c_{i,t})
    \;\propto\;
    \basespeaker{i}(u_{i,t} \mid z_{i,t}^{\star}, c_{i,t})
    \cdot
    \tomlistener{i}(z_{i,t}^{\star} \mid u_{i,t}, c_{i,t}),
\label{eqn:director_bps}
\end{align}
where $\tomlistener{i}$ models how the Builder interprets 
$u_{i,t}$ given the already-communicated context. However, 
the goal of each Director $D_i$ is not merely to produce 
an utterance consistent with its private view $o_{i,t}$, 
but to select $u_{i,t}$ that maximizes its marginal 
contribution to a \emph{joint} listener aggregating all 
three Directors---resolving the Builder's remaining 
uncertainty about $z^\star$ given 
$\{u_{j,t}\}_{j \neq i} \subset c_{i,t}$, such that the 
collective output moves $s_t$ progressively toward the ground truth 
target $\mathcal{T}$.

\begin{definition}[Joint ToM Listener]
\label{def:joint_tom}
The \emph{joint ToM listener} for the Builder is:
\begin{align}
    \tomlistener{\mathrm{joint}}
        \!\left(z^{\star} \mid u_1, u_2, u_3, c\right)
    \;\propto\;
    \exp\!\left(
        \sum_{i=1}^{3} \lambda_i\, R_i(u_{i}, s_t, \mathcal{T})
    \right),
\label{eqn:joint_tom}
\end{align}
where $R_i(u_i, s_t, \mathcal{T})$ is a reward signal measuring the
downstream task progress attributable to Director $D_i$'s utterance
$u_{i,t}$, and $\lambda_i \geq 0$ are weighting coefficients.
\end{definition}

This leads to the main theoretical result of this section.

\begin{theorem}[\textsc{CRAFT} as a Multi-Sender BPS]
\label{thm:craft_bps}
Under Eq.~\ref{eqn:bps} and Eq.~\ref{eqn:joint_tom}, let
$\basespeaker{i}(\cdot \mid z_i^{\star}, c_i)$ be the base speaker
of director $D_i$ and let $z^{\star} = (z_1^{\star}, z_2^{\star},
z_3^{\star})$ be the joint intention vector, where $c_i = (o_{i,t}, h_t, \{u_{j,t}\}_{j \neq i})$ is director
$D_i$'s private context and
$c = (h_t, s_t)$ is the shared public context available to all agents.
The optimal joint Director policy is equivalent to a \textbf{multi-sender Bounded Pragmatic Speaker}. Formally, the optimal joint Director policy satisfies:
\begin{align}
    \pi^{\star}(u_1, u_2, u_3 \mid z^{\star}, c)
    \;\propto\;
    \left(\prod_{i=1}^{3}
    \basespeaker{i}(u_i \mid z_i^{\star}, c_i)\right)
    \cdot
    \tomlistener{\mathrm{joint}}(z^{\star} \mid u_1, u_2, u_3, c),
\label{eqn:craft_bps}
\end{align}
which reduces to the standard single-sender BPS 
(\Cref{def:bps}) under $N=1$ (see Appendix~\ref{app:proofs} for proofs.).
\end{theorem}

\autoref{thm:craft_bps} establishes that optimal Director behavior requires each $D_i$ to satisfy two simultaneous conditions: (i)~faithfully encoding its private observations into candidate 
utterances via $\basespeaker{i}$, and (ii)~calibrating those 
utterances to the \emph{joint} information need of the Builder via $\tomlistener{\mathrm{joint}}$---contributing exactly the information the Builder requires that has not already been 
provided by the other Directors.\footnote{This framing is consistent with Gricean cooperativity~\citep{grice1975logic}: 
the joint ToM listener's calibration requirement subsumes the Maxim of Quantity in a grounded multi-agent setting.}  Importantly, this condition is irreducible to any individual Director's output quality: a \textbf{communication failure} 
occurs when $\tomlistener{\mathrm{joint}}$ fails to place sufficient mass on $z^\star$ despite it being reachable\footnote{In \craft{}, reachability is instantiated by 
oracle-verified moves---physically valid actions 
guaranteed to make forward progress from $s_t$ toward 
$\mathcal{T}$ (Sec.~\ref{ssec:structure_experiments}).}, regardless of how well any individual Director reasoned or communicated.

\vspace{-2mm}



\paragraph{Failure Modes in Multi-Sender BPS}
\citet{nguyen2024languagemodelsboundedpragmatic} identifies three failure modes for a single-sender BPS: \textbf{F1}~({\it limited 
search}), \textbf{F2}~({\it flawed pragmatics}), and \textbf{F3}~({\it inefficient inference}). In \craft{}, F1 occurs when $\basespeaker{i}$ fails to identify the correct missing block 
from $D_i$'s visible wall; F3 occurs when correct private reasoning fails to transfer into a sufficiently specified public message. Both 
remain agent-local. F2 becomes strictly harder in the multi-sender setting: each Director must calibrate against $\tomlistener{\mathrm{joint}}$ rather than a single listener, accounting for what all other Directors have already communicated. As implied in \Cref{thm:craft_bps}, a Director may avoid F1--F3 individually yet still contribute to a 
group-level failure---the irreducible fourth failure mode motivating the diagnostic framework in \Cref{ssec:evaluation_strategy}.



\vspace*{-3mm}
\section{Agent Architecture with LLMs}
\label{sec:agents}
\vspace*{-2mm}


\begin{figure}[t]
  \centering
    \includegraphics[width=\linewidth]{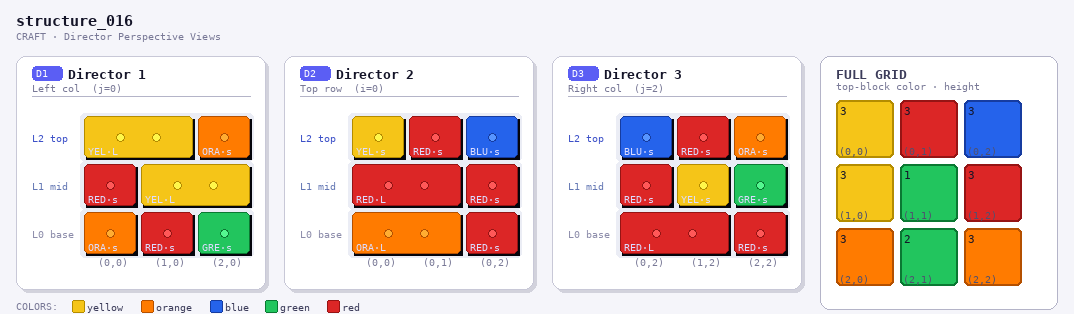}
\vspace*{-2mm}
\caption{Director perspective views for a 25 block structure with the ``camera'' at the ``bottom'' of the {\it Full Grid} view. D1 (left wall), D2 (far wall), and D3 (right wall) each represent a fixed 2D projection across all vertical layers. The Full Grid minimap shows 
ground-truth stack heights.
\vspace*{-2mm}
}

\label{fig:director_views}
\end{figure}

\subsection{Role, Task and Agent Setup}
\label{subsec:task_setup}
\vspace*{-2mm}
\an{\craft{} instantiates a \textbf{pragmatic speaker} 
role (Directors) and a \textbf{pragmatic listener} 
role (Builder), directly reflecting the multi-sender 
BPS framework (Sec.~\ref{sec:directors_as_bps}) and 
the well-established asymmetry between pragmatic 
production and comprehension~\citep{ferreira2005psycholinguistics, 
sieker2026hypocritical}: Directors must 
\emph{produce} pragmatically calibrated instructions 
under partial observability---the harder and less 
reliable capability in LLMs---while the Builder 
\emph{interprets} and acts on those instructions, 
a task LLMs handle more reliably. This is reflected in the Director's outputs which are full natural language utterances focusing on pragmatic production while the Builder has a restricted action space (e.g., selection of moves) focusing on pragmatic comprehension.  }
\vspace{-3mm}
\paragraph{View Geometry} As shown in Fig.~\ref{fig:director_views}, \textbf{three Director agents} (D1, D2 and D3) and \textbf{one Builder agent} collaborate in a “turn-by-turn” synchronous manner~\citep{ivison2023camels, nath2025collaborate} to 
reconstruct the target 3D structure on a $3{\times}3$ grid, where each ``cell'' 
holds a stack of up to three small or large (\textit{domino}) blocks in one of five colors. Each Director has a private 2D projection of the 
target corresponding to one wall of the structure (see Fig.~\ref{fig:director_views}), while the Builder observes only the current 
board state and the Directors' natural language messages. 
Private views overlap only at corners shared by two Directors, so \textit{no individual} can 
reconstruct the full target unilaterally, and the same physical block can appear 
as a different size to different Directors depending on whether both cells of a domino fall within their projection (cf. Fig.~\ref{fig:director_views}, D1 vs. D2). Success therefore requires each Director to reason 
about what information they uniquely hold and 
calibrate their public utterances accordingly (Sec.~\ref{sec:directors_as_bps}).




\vspace*{-4mm}
\subsection{Director Agent}
\label{subsec:director}
\vspace*{-3mm}

Each Director agent receives its private target view, the current board state, and the shared conversation history, and produces two outputs per turn: a private \texttt{<think>} block containing unconstrained spatial reasoning, and a public \texttt{<message>}. This 
two-part structure directly operationalizes the BPS base speaker and ToM listener distinction: the 
\texttt{<think>} block is where the director identifies missing blocks from its visible wall ($\basespeaker{i}$), and the public \texttt{<message>} is where that reasoning 
is distilled into an utterance calibrated given what the Builder and other Directors already know ($\tomlistener{i}$). Crucially, this separation allows \craft{} to diagnose whether a Director reasoned correctly but failed to transfer that into a sufficiently specified public message---a failure mode invisible to single-output 
architectures and directly measurable here through the gap between private reasoning quality and public communication quality. The Director prompt and its constituent parts are given in Figs.~\ref{fig:director_initial_prompt}--\ref{fig:director_initial_prompt3} in Appendix~\ref{app:agent_prompts}.





\vspace*{-3mm}
\subsection{Builder Agent}
\label{subsec:builder}
\vspace*{-2mm}

The Builder agent observes Director messages and conversation history each turn and executes \textbf{one} block placement or removal, or makes a request for clarification. It has no access to $\mathcal{T}$ and, by design, no independent goal toward $z^\star$---its sole task is to 
interpret Director descriptions as a pragmatic 
listener. However, in this multi-agent LLM setting, the Builder's 
own reasoning capacity introduces potential variance independent of Director communication quality. Oracle-verified candidates grounded in $\mathcal{T}$ address this by providing the Builder with pre-computed forward-progress moves from $s_t$, removing the need for independent reasoning about $z^\star$: the Builder selects among verified candidates rather than inferring reachability itself, directly operationalizing the reachability 
condition under which communication failure is defined (Sec.~\ref{sec:directors_as_bps}). Since oracle moves reflect \textit{locally} verified progress from $s_t$ rather than a \textit{globally} optimal trajectory toward $\mathcal{T}$, 
Directors must still coordinate coherently across turns for task completion to accumulate. Figs.~\ref{fig:builder_initial_prompt}--\ref{fig:builder_initial_prompt3} in Appendix~\ref{app:agent_prompts} show the full Builder prompts. Appendix~\ref{app:structure_and_oracle} provides experimental details on oracle candidates.

\vspace*{-2mm}
\section{Experiments}
\label{sec:experiments}
\vspace*{-2mm}


 \subsection{Generating Target 3D Structures}
 \label{ssec:structure_generation}
\vspace*{-2mm}

Target structures are generated by first assigning stack heights to grid cells. There are seven \textit{required} cells around the edges of the board, excluding the ``bottom edge'' (Fig.~\ref{fig:director_views}), which have 3 layers of blocks. The remaining 2 cells receive a uniformly sampled height $0-2$. Each layer is then tiled according to the calculated height at each cell using mix of small or large blocks. A large block occupies two adjacent cells \emph{on the same layer}, forming a \textit{domino} pair. Small blocks occupy a single position. Colors are sampled uniformly from the five options. Appendix~\ref{ssec:generation} provides details.
 
\vspace*{-3mm}
 \subsection{Collaborative Construction Gameplay}
 \label{ssec:structure_experiments}
\vspace*{-2mm}

We evaluate LLMs acting as the Directors and the Builder building 20 structures drawn from 
our evaluation set, spanning 7 simple, 8 medium, and 5 complex 
configurations with block counts ranging from 21 to 25 (mean 23.2). 
Every game starts from an empty board.
Each Director is assigned a 
personality archetype~\citep{jung2023archetypes} or “persona”~\citep{sun2024buildingbetteraiagents}. This is done to meaningfully increase diversity in Director utterances so that they do not converge to similar templatic statements, and are assigned deterministically to ensure consistent roles across all model evaluations. A game for each model-structure combination consists of 20 communicative turns, each of which consists of collecting responses from a randomized 
selection of one to three unique Directors sequentially, followed by a Builder move 
selection conditioned on the Director discussion that updates the current board (See Fig.~\ref{fig:craft_overview}). This reflects natural multi-party 
conversation dynamics~\citep{ganesh2023survey}, where Director participation 
varies per turn---on average two Directors contribute per turn. Additionally, in each turn, the current conversation history of all prior Director responses is shown to the Directors, while the Builder only observes the current turn's Director responses. For details on structures and Director personalities, see \Cref{tab:archetypes} in Appendix~\ref{app:structure_and_oracle}.

\vspace*{-3mm}
\subsection{Evaluation Strategy}
\label{ssec:evaluation_strategy}
\vspace*{-2mm}

We report metrics at turn 20 that measure how close the final board state is to the target structure. The most important are {\bf completion percentage}, or percentage of blocks whose color and position exactly match those of the final structure; and {\bf overall progress}, the unweighted mean of completion percentage, IoU of the current vs. target structure, and accuracy of block {\it colors} in each cell independent of vertical layer order. 
We additionally report remove gap ($\text{Gap} = \texttt{REMOVE} -
\text{oracle remove rate}$), averaged across turns. A positive
gap indicates Directors instructed more removals than the board required. Appendix~\ref{subsec:metrics} provides more details on task metrics. 

\vspace*{-2mm}
\paragraph{Automatic Grading} 
A pragmatically competent Director produces a message that is not a
transcription of its internal reasoning but a selective, non-redundant
instruction calibrated to current shared knowledge. This motivates three diagnostically-independent  \textit{multidimensional} LLM “graders”~\citep{zheng2023judging,davidson2025collaborationgap}.\footnote{Due to the large scale experiments in our evaluations involving thousands of trajectory logs with long prompts, human evaluation was not feasible. } The \textbf{Spatial Grounding} (SG) judge evaluates whether each Director's
private {\tt <think>} block correctly
identifies the missing block, its layer, size, and physical
executability (F1; prompt with questions in Fig.~\ref{fig:spatial_judge_prompt}).
The \textbf{Mind Modeling}\footnote{We use ``Mind Model'' in the sense of  evaluating $\tomlistener{\mathrm{joint}}$ qualities in Director 
messages~\citep{riemer2024position, xu2024walking}, not as a claim 
about functional ToM capacities~\citep{mahowald2024dissociatinglanguagethoughtlarge}.} (MM) judge evaluates whether the Director's public message adds novel information,
leverages the Director's unique wall perspective, and acknowledges
conflicts with other directors (F2;
Fig.~\ref{fig:mind_model_judge_prompt}). 
The SG and MM judges measure the quality of a Director's individual reasoning and public message, but a communication failure (Sec.~\ref{sec:directors_as_bps}) can still happen if the \textit{collective} Director output was insufficient for a rational Builder to identify an oracle move.  As such, the \textbf{Pragmatic Sufficiency} (PS) judge evaluates 
whether the collective Director output was sufficient for 
a rational Builder to identify at least one oracle-correct 
move---the empirical instantiation of $\tomlistener{\mathrm{joint}}$ 
placing sufficient mass on $z^\star$ 
(Sec.~\ref{sec:directors_as_bps}; 
Fig.~\ref{fig:pragmatic_judge_prompt}), a property 
irreducible to any individual Director's utterance. 





We score each grader question by mapping the model response to a
value: \textit{Yes} $= 1$, \textit{No} $= 0$, and \textit{Unclear} $=
0.5$. Overall SG and MM scores are the mean across all questions for a
given director turn; the PS score is the mean across applicable
questions for a given collective turn. All scores are averaged across
\textbf{three} independent grader runs to reduce variance from stochastic model
outputs. All judge prompts with questions are provided in Figs.~\ref{fig:spatial_judge_prompt}--\ref{fig:pragmatic_judge_prompt} (Appendix~\ref{app:llm_judge_prompts}).

\vspace*{-2mm}
\paragraph{Models and Implementation}
We evaluate 15 Director models. Eight are open-weight 7B--72B parameter models: Qwen-2.5-Instruct 7B/14B/32B/72B~\citep{qwen2025qwen25technicalreport}, 
Llama-3.1-8B Instruct~\citep{llama3modelcard}, Mistral-7B-Instruct-v0.3~\citep{jiang2023mistral}, 
Gemma 2 9B Instruct~\citep{team2024gemma}, DeepSeek-V2-Lite~\citep{shao2024deepseekmath}. Seven are frontier/proprietary models: GPT-4o, GPT-4o-Mini, GPT-4.1-Mini~\citep{openai2023gpt4}, 
Claude-Sonnet-4.6, Gemini-2.5-Flash, 
Gemini-3-Flash, Gemini-3.1-Flash-Lite.\footnote{For brevity, we refer to models by shortened names hereafter 
(e.g., Qwen-7B for Qwen-2.5-Instruct-7B, Mistral-7B for Mistral-7B-Instruct-v0.3, Gemini-3-Flash for 
Gemini-3-Flash-Preview).} All are paired with a fixed GPT-4o-mini~\citep{openai2024gpt4ocard} Builder for all experiments. 
Each model is evaluated 3 times independently on 20 test structures over 20 turns. Open-weight models use a 512-token output budget; 
frontier models use 2{,}000 tokens (GPT series) or 3{,}000 tokens 
(Claude, Gemini) to accommodate extended chain-of-thought generation 
without truncating reasoning traces. For the LLM-grading evaluation (Sec.~\ref{ssec:evaluation_strategy}), we use GPT-4o-mini. Appendix~\ref{app:exp-config} contains additional notes on experimental configurations. 

\vspace*{-3mm}
\section{Results}
\label{sec:results}
\vspace*{-2mm}


\begin{table*}[t]
\centering
\small
\setlength{\tabcolsep}{2.5pt} 
\begin{threeparttable}
\resizebox{\textwidth}{!}{
\begin{tabular}{lccccccc}
\toprule
\textbf{Model} 
    & \textbf{Prog.}$\uparrow$ 
    & \textbf{Comp.}$\uparrow$ 
    & \textbf{Pos.\ Acc.}$\uparrow$ 
    & \textbf{IoU}$\uparrow$ 
    & \textbf{Fail}$\downarrow$
    & \textbf{\texttt{REMOVE}}$\downarrow$
    & \textbf{Gap}$\downarrow$ \\
\midrule




Gemini-3-Flash    & $\mathbf{0.675}_{{\pm0.052}}$ & $\mathbf{0.716}_{{\pm0.045}}$ & $\mathbf{0.594}_{{\pm0.065}}$ & $\mathbf{0.817}_{{\pm0.034}}$ & $0.625_{{\pm0.125}}$ & $\mathbf{0.196}$ & $\mathbf{ 0.018}$ \\

GPT-4o            & $0.588_{{\pm0.051}}$ & $0.633_{{\pm0.048}}$ & $0.500_{{\pm0.061}}$ & $0.753_{{\pm0.038}}$ & $0.421_{{\pm0.116}}$ & $0.280$ & $ 0.056$ \\

GPT-4o-Mini       & $0.333_{{\pm0.041}}$ & $0.383_{{\pm0.040}}$ & $\underline{0.233}_{{\pm0.047}}$ & $0.531_{{\pm0.041}}$ & $0.550_{{\pm0.114}}$ & $0.432$ & $ 0.254$ \\

GPT-4.1-Mini      & $0.312_{{\pm0.053}}$ & $0.352_{{\pm0.054}}$ & $0.233_{{\pm0.056}}$ & $0.481_{{\pm0.053}}$ & $0.500_{{\pm0.115}}$ & $0.463$ & ${0.388}$ \\

Claude-Sonnet-4.6 & $0.285_{{\pm0.036}}$ & $0.332_{{\pm0.038}}$ & $0.189_{{\pm0.039}}$ & $0.479_{{\pm0.041}}$ & $0.350_{{\pm0.109}}$ & $0.395$ & $0.265$ \\
Gemini-2.5-Flash  & $\underline{0.257}_{{\pm0.030}}$ & $0.279_{{\pm0.033 }}$ & $0.206_{{\pm0.033}}$ & $0.428_{{\pm0.039}}$ & $\mathbf{0.300}_{{\pm0.105}}$ & $0.467$ & $0.402$ \\

Gemini-3.1-Flash-lite   & $0.257_{{\pm0.052}}$ & $\underline{0.273}_{{\pm0.055}}$ & $0.286_{{\pm0.054}}$ & $\underline{0.211}_{{\pm0.0514}}$ & $0.550_{{\pm0.114}}$ & $\underline{0.540}$ & $\underline{0.467}$ \\
\midrule
Mistral-7B        & $\mathbf{0.631}_{{\pm0.053}}$ & $\mathbf{0.673}_{{\pm0.046}}$ & $\mathbf{0.539}_{{\pm0.067}}$ & $\mathbf{0.793}_{{\pm0.033}}$ & $0.500_{{\pm0.129}}$ & $\mathbf{0.124}$ & $\mathbf{-0.124}$ \\
Qwen-7B           & $0.612_{{\pm0.044}}$ & $0.665_{{\pm0.040}}$ & $0.517_{{\pm0.052}}$ & $0.778_{{\pm0.031}}$ & $0.556_{{\pm0.121}}$ & $0.205$ & $-0.116$ \\
Llama-8B          & $0.586_{{\pm0.052}}$ & $0.630_{{\pm0.049}}$ & $0.506_{{\pm0.056}}$ & $0.741_{{\pm0.057}}$ & $0.684_{{\pm0.110}}$ & $0.277$ & $0.080$ \\
Gemma-9B          & $0.578_{{\pm0.044}}$ & $0.628_{{\pm0.039}}$ & $0.483_{{\pm0.054}}$ & $0.751_{{\pm0.033}}$ & $0.600_{{\pm0.112}}$ & $0.122$ & $ -0.084$ \\
Qwen-72B          & $0.557_{{\pm0.049}}$ & $0.606_{{\pm0.042}}$ & $0.461_{{\pm0.064}}$ & $0.733_{{\pm0.036}}$ & $0.421_{{\pm0.116}}$ & $0.245$ & $0.047$ \\

Qwen-14B          & $0.476_{{\pm0.063}}$ & $0.514_{{\pm0.061}}$ & $0.394_{{\pm0.068}}$ & $0.642_{{\pm0.055}}$ & $0.611_{{\pm0.118}}$ & $0.355$ & $0.229$ \\

DeepSeek-Lite     & $0.419_{{\pm0.041}}$ & $0.474_{{\pm0.039}}$ & $0.317_{{\pm0.048}}$ & $0.617_{{\pm0.038}}$ & $\mathbf{0.400}_{{\pm0.112}}$ & $0.148$ & $-0.275$ \\

Qwen-32B          & $0.339_{{\pm0.045}}$ & $0.378_{{\pm0.048}}$ & $0.250_{{\pm0.046}}$ & $0.530_{{\pm0.048}}$ & $\underline{0.850}_{{\pm0.082}}$ & $0.448$ & $\underline{0.374}$ \\
\bottomrule
\end{tabular}}
\end{threeparttable}
\vspace*{-2mm}
\caption{Task performance in \craft{} across frontier models (top) and open-weight models (bottom) across 20 games after turn $t=20$. The Builder receives up to 5
“legal” moves per turn and selects among them based solely on
Director instructions. \textit{Prog.}, \textit{Comp.}, \textit{Pos.\ Acc.}, and \textit{IoU}\
are continuous metrics; \textit{Failed} is a binary per-turn outcome. Subscripted values show Mean$_{{\pm\text{SEM}}}$. \texttt{REMOVE} and Gap are aggregated over all turns. \textbf{Bold} = best within group. \underline{Underline} = worst across all models.
\vspace*{-2mm}
 }
\label{tab:oracle_performance}
\end{table*}


 
  
 \begin{wrapfigure}{l}{0.45\textwidth}
 
    \centering
     \vspace{-4mm}
    \includegraphics[width=\linewidth]{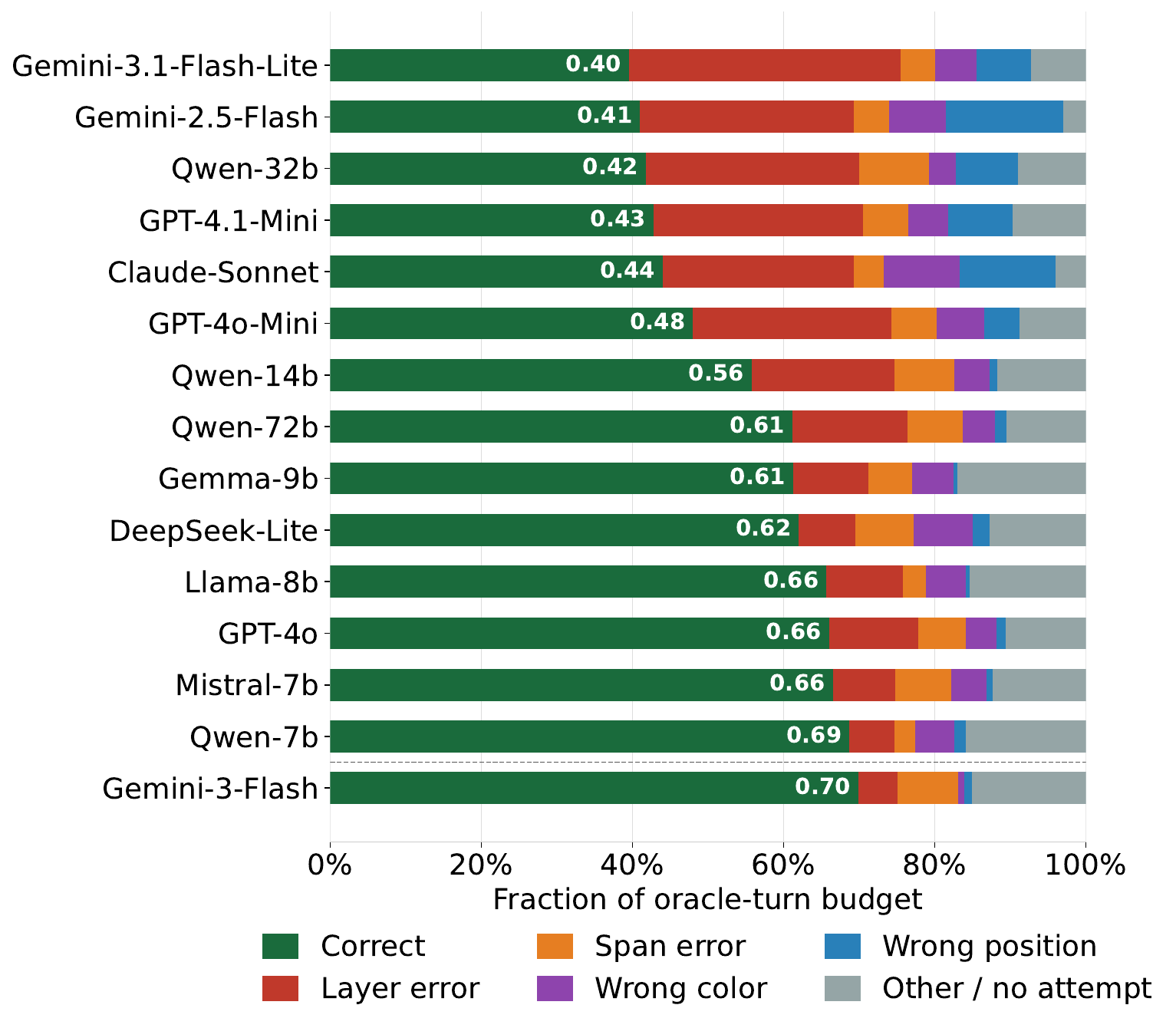}
     \vspace{-2mm}
    \caption{Failure taxonomy over all turns across 15 director models.
  }
    \label{fig:failure_taxonomy}
      \vspace{-4mm}
\end{wrapfigure}

\paragraph{Frontier models do not uniformly dominate open-weight models} Table~\ref{tab:oracle_performance} shows substantial variation within both groups. Gemini-3-Flash leads all 15 models with average progress of $0.675$, but GPT-4.1-Mini ($0.312$), Claude-Sonnet-4.6 ($0.285$), and both other Gemini-Flash variants ($0.257$) fall below most open-weight models. This shows that perspective-based spatial reasoning can be challenging even for frontier models, consistent with prior work~\citep{li2025viewspatialbenchevaluatingmultiperspectivespatial}.

\vspace*{-2mm}
\paragraph{Move-Level Failures}
Fig.~\ref{fig:failure_taxonomy} shows a taxonomy of failures across all 15 
models, computed by replaying saved game logs and classifying each turn in which oracle-verified moves were available\footnote{Turns without valid oracle moves rarely exceed $15\%$ of total turns across all models on average. \Cref{fig:scalar_outcomes} provides details.} by its dominant failure mode. Layer errors (red) dominate across nearly all models, 
confirming that 3D layer inference is the primary move-level 
bottleneck---Directors likely correctly identify target positions but 
specify a vertical layer where the mode is invalid and fails. Wrong position errors (blue) indicate that Directors like Claude-Sonnet and the two lowest-performing Gemini models likely issue instructions that map to incorrect 
board locations. Span errors (orange) are 
elevated for Gemini-3-Flash relative to its otherwise strong 
performance: despite Directors correctly identifying block type and 
position, the Builder fails to infer the required second-cell 
endpoint for large blocks from the Director instructions. 
This is compounded by the partial observability design---a large 
block spanning two cells appears as size~2 to the Director whose 
view contains both cells, but as size~1 to Directors who see only 
one face, making it difficult for Directors to consistently 
communicate domino placement with sufficient precision for the 
Builder to resolve the span 
(see Appendix~\ref{ssec:encoding} for 
details on block encoding and partial view projection).

\vspace*{-2mm}

\paragraph{Correction spirals as a signature of communication failure}
The \textit{remove gap}, or the difference between the fraction of turns where Directors instruct a remove action and the fraction where the oracle prescribes one, separates frontier from open-weight models most clearly and is strongly negatively correlated with oracle adherence ($r = -0.543$, $p < 0.001$), and therefore task progress. 
Fig.~\ref{fig:remove_evolution} plots the \textit{evolution} of builder 
remove actions across turns against oracle-prescribed remove 
actions across all evaluations.\footnote{Full per-model evolution plots for all base and frontier models 
appear in Figs.~\ref{fig:remove_evolution_base} 
and~\ref{fig:remove_evolution_frontier} in 
Appendix~\ref{app:additional_experimental_results}.
} While $95.4\%$ of oracle-prescribed moves are placements, the fraction of turns 
requiring a removal grows sharply after turn 10 as boards accumulate 
errors.
Low-performing models show over-removals early in the game before any board errors have accumulated---indicating that 
over-removal is a proactive Director communication pattern rather 
than a reactive response to mistakes.
This ``correction spiral'' reflects a 
fundamental property of 
the multi-perspective coordination 
requirement: three Directors may simultaneously identify 
different blocks to remove, leaving the Builder to resolve 
conflicting priorities across turns without guaranteed progress.  In contrast, well-performing 
models (Llama-8B, Gemini-3-Flash) track the oracle remove signal 
closely throughout the game. 

\begin{figure}[t!]
  \centering
  \includegraphics[width=.9\linewidth]{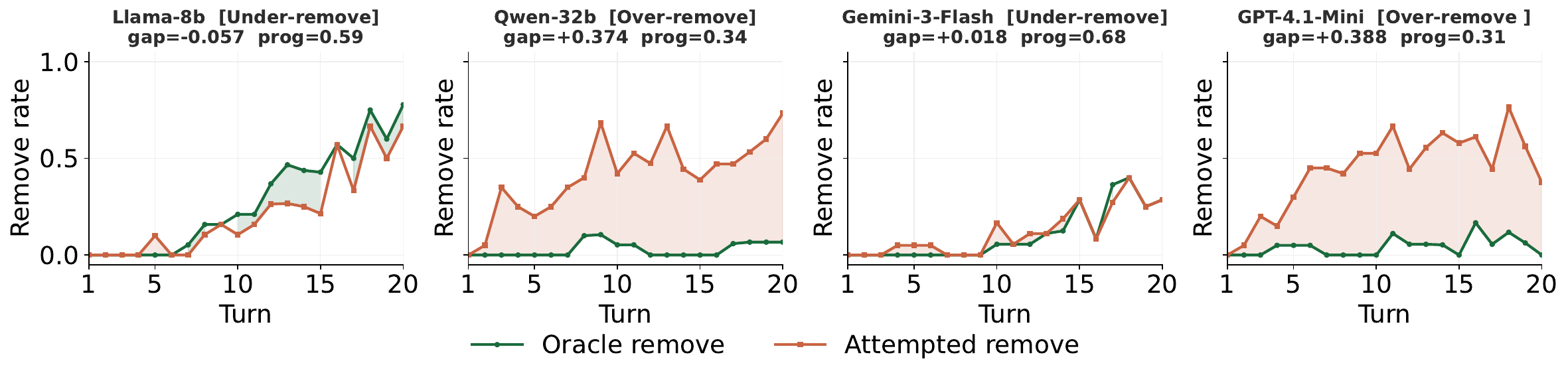}
\vspace*{-2mm}
  \caption{Per turn oracle-prescribed vs.\ attempted remove rate, averaged
           across all 20 structures (shading = {\it remove gap}, or difference between fraction of turns where Directors instruct a remove action and fraction where oracle prescribes one; \textcolor{darkgreen}{\bf green}: model under-removes, \textcolor{saddlebrown}{\bf orange}: model over-removes).
           Subplot titles show mean gap and final-turn task progress. 
\vspace*{-2mm}
      }
  \label{fig:remove_evolution}
\end{figure}

\begin{figure}[t!]
  \centering
  \includegraphics[width=\linewidth]{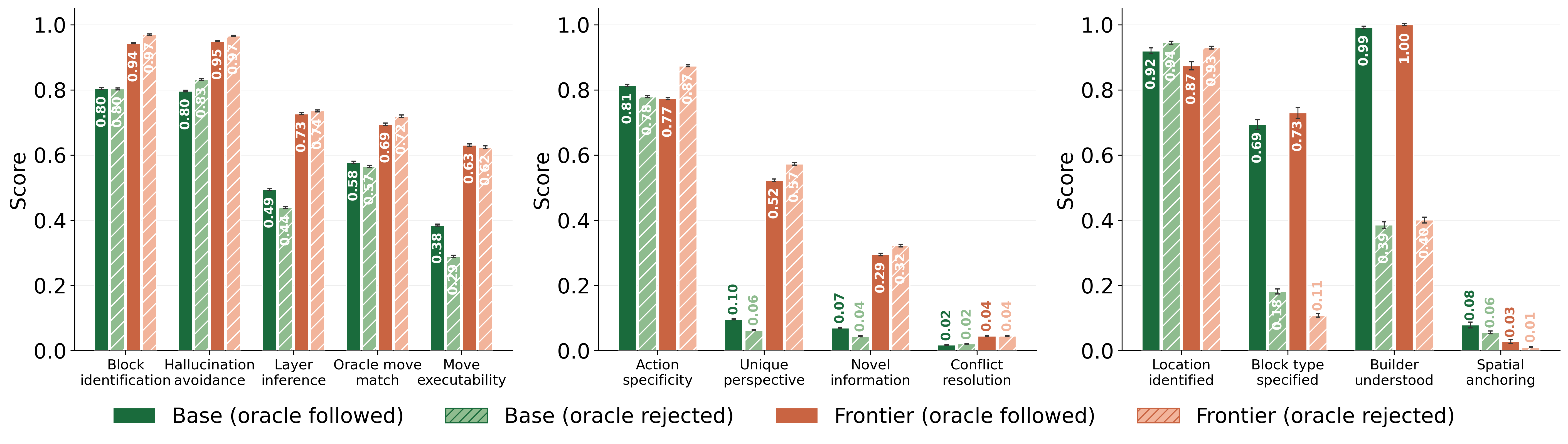}
\vspace*{-2mm}
\caption{LLM grader scores across \textbf{three} evaluation dimensions---\textbf{spatial
grounding} (left), \textbf{mind modeling} (center), and \textbf{pragmatic sufficiency}
(right)---broken down by question and model group. Error bars denote standard error of the mean across all 
structure--turn--director observations per model from 3 independent runs.}

  \label{fig:judge_questions}
\vspace*{-2mm}
\end{figure}

\vspace*{-3mm}
 \subsection{LLM Grader Results}
\label{ssec:llm_judge_results}
\vspace*{-2mm}


Oracle adherence 
and remove behavior characterize \emph{what} goes wrong and \emph{how often}, but do not explain \emph{why}. We therefore further investigate communication failures with LLM graders. 

\vspace*{-2mm}
\paragraph{Individual reasoning and message quality do not distinguish turn or task success} Fig.~\ref{fig:judge_questions} shows automatic grading results\footnote{Detailed results are shown in \Cref{tab:judge_scores} and Fig.~\ref{fig:judge_permodel} in the Appendix.} for \textbf{spatial grounding} (left), \textbf{mind modeling} (middle), and \textbf{pragmatic sufficiency} (right) of Director internal reasoning and public messages, grouped by Builder oracle adherence. This reveals a key turn-level diagnostic: SG and MM scores are nearly identical given oracle adherence condition, indicating that Directors reason and communicate at similar quality regardless of whether the Builder ultimately follows the oracle. In contrast, PS scores drop sharply when the Builder does not execute the oracle-recommended move. This suggests that individual communication quality (SG and MM scores) does not distinguish successful from unsuccessful turns, whereas collective sufficiency does---directly operationalizing the group-level irreducibility established in Sec.~\ref{sec:directors_as_bps}. This is expected: the PS judge evaluates the joint Director output as a whole, which is precisely the object whose sufficiency determines whether the Builder can identify a correct move, regardless of any individual Director's communication quality.


Frontier models score higher on spatial grounding 
($0.829{\pm}0.036$ vs. $0.658{\pm}0.069$) and mind modeling 
($0.642{\pm}0.067$ vs. $0.502{\pm}0.024$), yet achieve lower task 
progress ($0.387$ vs. $0.525$). 
Better individual reasoning and message quality, in other words, does 
not automatically produce better task outcomes. This calls into question the utility of using benchmarks of individual communication for multi-agent tasks; higher-level qualities of the information exchange are more explanatory.
Table~\ref{tab:judge_scores} in Appendix~\ref{app:llm_judge_prompts} reports per-model judge scores with standard errors.

\vspace*{-2mm}
\paragraph{Individual Director Communication vs. Collaborative Task Performance}

We correlated all judge dimensions with progress across all Directors.\footnote{Complete results shown in \Cref{tab:feature_correlations_judge} in Appendix~\ref{app:llm_judge_prompts} and ~\ref{app:mediation_analysis}.} Most show negative correlations with progress, but \textbf{unique perspective utilization (MM5)} stands out: it 
directly measures whether a Director reasons about what is “exclusively visible” from their wall, which is precisely what would lead a Director to identify board corrections invisible to others and instruct removals accordingly. 
Mediation analysis~\citep{mackinnon2007mediation} confirms the likely causal chain. Unique perspective utilization alone explains $R^2=0.330$ of progress variance, remove gap alone explains $R^2=0.609$, and both together explain $R^2=0.633$, a gain of only $0.024$. Controlling for remove gap establishes full mediation: the partial correlation of unique perspective with progress is non-significant ($r=-0.247$, $p=0.374$). \textbf{Directors who carefully leverage their private wall view produce more correction-oriented instructions, which drive over-removal, which halts progress over the 20 turns.} This suggests that the very sophistication that makes frontier Directors better individual communicators makes them relatively worse collective coordinators in \craft{}'s multi-sender setting.

\vspace*{-3mm}
\section{Conclusion}
\label{sec:conc}
\vspace*{-2mm}



\an{
We introduced \textsc{CRAFT}, a multi-agent benchmark for coordination under asymmetric information, grounded in a multi-sender Bounded Pragmatic Speaker (BPS) formulation in which optimal behavior requires calibrating to a \emph{joint} listener. Empirically, we find a clear dissociation between individual communicative competence and collective task success: frontier models score higher on communication metrics yet underperform smaller open-weight models on task completion. This gap reflects a frontier-specific failure mode consistent with our theory---Directors that better leverage their unique perspective produce correction-heavy instructions that satisfy individual ToM objectives but misalign with the joint listener, inducing over-removal and exhausting the turn budget. Open-weight models, despite weaker individual competence, avoid this failure due to less aggressive epistemic updating. No model successfully models the joint 
listener in practice. These results show that individual pragmatic competence is necessary but not sufficient for coordination, and that failure to approximate the joint listener leads to systematic breakdowns. More broadly, \textsc{CRAFT} demonstrates that benchmarks of individual pragmatic reasoning are insufficient for evaluating multi-agent systems: coordination and individual competence are empirically dissociable and can move in opposite directions.}

\vspace*{-2mm}
\paragraph{Limitations and Future Work}
\textsc{CRAFT} is evaluated in a \textit{text-only} setting and does not consider models with direct access to visual inputs~\citep{li2025viewspatialbenchevaluatingmultiperspectivespatial, liu2025spatialreasoningmultimodallarge}. While multimodal inputs may improve individual spatial grounding, it remains unclear whether such gains transfer to multi-agent coordination.
Our experiments also use a fixed Builder with access to oracle-assisted candidate moves. This controlled design isolates the effect of Director communication and enables clearer credit assignment, but does not fully reflect real-world settings where agents must act without access to ground-truth guidance or may need to “explore” possible moves.
We also do not explore settings with heterogeneous Director models (e.g., mixing open and proprietary models within a game)~\citep{davidson2025collaborationgap}. While we empirically validated the credit attributed to each Director in the \textit{joint listener} for the Builder (Eq.~\ref{eqn:joint_tom}) for \craft{} games in Appendix~\ref{app:lambda_consistency_validation}, heterogeneous agents in the loop would naturally require more sophisticated methods of credit assignment from MARL research such as counterfactual credits~\citep{deshmukh2026capocounterfactualcreditassignment}, or more principled approaches like Shapley-Owen values~\citep{shapley1953stochastic, nath2026owen}. Such configurations could provide further insight into how differences in alignment algorithms, pretraining, and post-training data influence collaborative performance and agent behavior, including partner-aware coordination in information-asymmetric settings~\citep{curvo2025traitors, liang2025llmhanabievaluatingmultiagentgameplays, hu2021offbelieflearning, nath2025learning}. While our controlled setup reduces variability and improves interpretability, it limits conclusions about joint optimization in heterogeneous multi-agent systems, where communication and execution policies may co-adapt.



\bibliography{colm2026_conference, rlhf, rlhf_2, nips_deception, nk_cites, nk_hannah_cites}


\appendix

\section{Proofs}
\label{app:proofs}

\begin{proof}[Proof]
We proceed in three steps in order to prove \Cref{thm:craft_bps}.

\paragraph{Step 1: Per-Director optimization.}
Each Director $D_i$ selects their utterance $u_i$ to maximize expected
utility under the joint listener.
Since the joint listener $\tomlistener{\mathrm{joint}}$ evaluates the
full tuple $(u_1, u_2, u_3)$, the optimal utterance for $D_i$ solves:
\begin{align}
    u_i^{\star}
    &= \argmax_{u_i \in \mathcal{U}}\;
    \log \basespeaker{i}(u_i \mid z_i^{\star}, c_i)
    + \log \tomlistener{\mathrm{joint}}
        (z^{\star} \mid u_1, u_2, u_3, c),
\label{eqn:per_director_opt}
\end{align}
where we work in log space for convenience.
This is the standard BPS objective (Eq.~\ref{eqn:bps}) applied to
$D_i$, with $\tomlistener{\mathrm{joint}}$ playing the role of the
ToM listener.

\paragraph{Step 2: Conditional independence of base speakers.}
Each base speaker $\basespeaker{i}$ is parameterized independently and therefore
it conditions on $D_i$'s private context $c_i = (o_{i,t}, h_t,
\{u_{j,t}\}_{j \neq i})$ and intention $z_i^{\star}$, and does not
depend on the utterances or intentions of the other directors except
through the shared conversation history $h_t$ already included in
$c_i$.
Formally, for $i \neq j$:
\begin{align}
    \basespeaker{i}(u_i \mid z_i^{\star}, c_i)
    \;\perp\!\!\!\perp\;
    \basespeaker{j}(u_j \mid z_j^{\star}, c_j)
    \;\Big|\; c,
\label{eqn:independence}
\end{align}
where $c = (h_t, s_t)$ is the shared public context.
Under this conditional independence, the joint distribution over all
three utterances given the shared context factorizes as:
\begin{align}
    p(u_1, u_2, u_3 \mid z^{\star}, c)
    \;=\;
    \prod_{i=1}^{3}
    \basespeaker{i}(u_i \mid z_i^{\star}, c_i).
\label{eqn:joint_factorisation}
\end{align}

\paragraph{Step 3: Deriving the joint policy.}
Combining Eq.~\ref{eqn:joint_factorisation} with the joint listener
$\tomlistener{\mathrm{joint}}$ (Eq.~\ref{eqn:joint_tom}), the optimal
joint policy is obtained by multiplying the factorized base speaker
distribution by the joint listener and normalizing:
\begin{align}
    \pi^{\star}(u_1, u_2, u_3 \mid z^{\star}, c)
    &\;\propto\;
    p(u_1, u_2, u_3 \mid z^{\star}, c)
    \cdot
    \tomlistener{\mathrm{joint}}(z^{\star} \mid u_1, u_2, u_3, c)
    \nonumber \\
    &\;=\;
    \left(\prod_{i=1}^{3}
    \basespeaker{i}(u_i \mid z_i^{\star}, c_i)\right)
    \cdot
    \tomlistener{\mathrm{joint}}(z^{\star} \mid u_1, u_2, u_3, c).
\label{eqn:joint_policy_derived}
\end{align}
This is exactly Eq.~\ref{eqn:craft_bps}, completing the derivation.

To verify that Eq~\ref{eqn:joint_policy_derived} preserves the BPS
form, substitute Eq.~\ref{eqn:joint_tom} for
$\tomlistener{\mathrm{joint}}$:
\begin{align}
    \pi^{\star}(u_1, u_2, u_3 \mid z^{\star}, c)
    &\;\propto\;
    \left(\prod_{i=1}^{3}
    \basespeaker{i}(u_i \mid z_i^{\star}, c_i)\right)
    \cdot
    \exp\!\left(\sum_{i=1}^{3} \lambda_i R_i(u_i, s_t,
    \mathcal{T})\right)
    \nonumber \\
    &\;=\;
    \prod_{i=1}^{3}
    \left[
    \basespeaker{i}(u_i \mid z_i^{\star}, c_i)
    \cdot
    \exp\!\left(\lambda_i R_i(u_i, s_t, \mathcal{T})\right)
    \right],
\label{eqn:per_director_bps_form}
\end{align}
where the last step uses the fact that the exponential of a sum
factorizes into a product of exponentials.
Each factor in Eq.~\ref{eqn:per_director_bps_form} has exactly the
BPS form of Eq.~\ref{eqn:bps}, with $\basespeaker{i}$ as the base
speaker and $\exp(\lambda_i R_i / Z_i)$ as the per-director ToM
listener, confirming that the multi-sender joint policy is a product
of individual BPS policies.
\end{proof}

\begin{lemma}[Single-Sender BPS as Special Case]
\label{lemma:n1_special_case}
\autoref{thm:craft_bps} reduces to the standard single-sender 
BPS (\autoref{def:bps}) under $N=1$.
\end{lemma}

\begin{proof}
In order to show that the single-sender BPS framework is a special case of \Cref{thm:craft_bps}, we consider the setup where $N=1$. The private context of $D_1$ reduces to 
$c_1 = (o_{1,t}, h_t, \emptyset)$ since $\{u_{j,t}\}_{j \neq 1} 
= \emptyset$, and $z^\star = z_1^\star$. The joint ToM listener 
(\autoref{def:joint_tom}) reduces to:
\begin{align}
    \tomlistener{\mathrm{joint}}(z_1^\star \mid u_1, c_1)
    \;\propto\;
    \exp\!\left(\lambda_1 R_1(u_1, s_t, \mathcal{T})\right)
    \;=:\;
    \tomlistener{1}(z_1^\star \mid u_1, c_1),
\end{align}
since all terms $\lambda_i R_i = 0$ for $i > 1$. Substituting 
into \autoref{eqn:craft_bps}, the product over base speakers 
collapses to a single factor:
\begin{align}
    \pi^\star(u_1 \mid z_1^\star, c_1)
    \;\propto\;
    \basespeaker{1}(u_1 \mid z_1^\star, c_1)
    \cdot
    \tomlistener{1}(z_1^\star \mid u_1, c_1),
\end{align}
which is \autoref{eqn:bps} exactly. 
\end{proof}

\noindent \autoref{lemma:n1_special_case} suggests that the single-sender approximation is insufficient whenever Directors hold 
complementary information. In \autoref{fig:director_views}, no single Director observes the full structure: D1 sees the left wall, D2 the top row, D3 the right wall, with overlap only at shared corners. Communicative sufficiency for the Builder therefore requires 
that the \emph{collective} output $(u_1, u_2, u_3)$ covers what no individual Director can specify alone---a condition invisible to 
any single-sender evaluation and only measurable at the group level.

\paragraph{Connection Between Homogeneity of Director Groups and Equal Weighting in joint ToM listener (via $\lambda_i$ in Eq.~\ref{eqn:joint_tom}).}

In general, the weighting coefficients $\lambda_i$ in the joint ToM listener for the Builder agent capture how the Builder aggregates per-Director contributions and may reflect differences in capability, information asymmetry, or task geometry. In heterogeneous settings~\citep{davidson2025collaborationgap}, this can induce asymmetric weighting of these coefficients. In contrast, the homogeneous Director configuration evaluated in \craft{} is “permutation-invariant”: all Directors share the same underlying model and differ only in role and partial observation, and no viewpoint is structurally privileged. Under this symmetry, a consistent aggregation mechanism should assign equal weight to each Director’s contribution---a quality we empirically validate in \Cref{app:lambda_consistency_validation}.

\section{Oracle Settings, Structure Generation and Director View Computation}
\label{app:structure_and_oracle}
\subsection{Oracle Evaluation}
\label{sec:oracle}

\paragraph{Implementation}
The oracle condition operationalizes the reachability 
requirement established in Sec.~\ref{sec:directors_as_bps}: 
at each turn, we enumerate all moves that make verified 
forward progress from $s_t$ toward $\mathcal{T}$, covering three cases: (i)~\emph{placement}---if the current stack is shorter 
than the target, the next required block (including color, size, and for large blocks, the span partner cell) is generated as a candidate;
(ii)~\emph{excess removal}---if the current stack exceeds the target depth, a remove move for the topmost block is generated;
(iii)~\emph{wrong block correction}---if the stacks are the same depth but a block at some layer is incorrect, a remove move is generated for the topmost block if it is 
the wrong one, or for the correct block above it if buried. All candidates are verified by simulating execution against 
a copy of the current game state before inclusion. Only moves that succeed in simulation and produce correct structural placement are retained. \textbf{Up to $N{=}5$ verified candidates} are sampled per turn as the Builder's input observation for the main experiments reported here (\Cref{sec:results}) using a deterministic seed derived from the structure index and turn number, ensuring 
reproducibility across model comparisons. Large block candidates include the span partner cell explicitly, since 
both endpoints must be specified for valid execution.

\paragraph{Builder selection criterion}
Verified candidates are injected into the Builder's prompt 
alongside the Director discussion, presented in a 
lightweight natural language format---for example, 
\texttt{PLACE gs @ (0,1) layer 0} for a small block or 
\texttt{PLACE bl @ (1,0) layer 0 $\to$ (2,0)} for a large 
block spanning two cells. The Builder is instructed to 
select the candidate that it believes at least one Director 
is describing, based on the current turn's Director 
discussion. If no candidate clearly matches, the Builder 
issues a clarification request rather than selecting 
arbitrarily. This preserves the full pragmatic inference 
requirement: the Builder must still interpret Director 
natural language and map it to a specific candidate, 
consistent with its role as a neutral passive listener 
established in Sec.~\ref{subsec:builder}. The Builder's 
confirmation field is extended to include a brief rationale 
identifying which Director(s) were followed and whether 
Directors agreed or conflicted, providing signal for 
subsequent analysis.

\paragraph{Oracle as a local progress signal}
The oracle verifies forward progress from $s_t$ at each 
turn but has no lookahead: a locally correct move is not 
necessarily a strategically good one if it creates 
correction debt downstream. Whether the Builder follows 
oracle depends on whether Directors made the right 
candidate identifiable; whether locally correct moves 
accumulate into global completion depends on whether 
Directors maintained a coherent group strategy across 
turns rather than consuming the turn budget on redundant 
or misdirected instructions. Oracle adherence is 
therefore necessary but not sufficient for task 
completion---and the inversion we observe 
(Sec.~\ref{sec:results}), where frontier models with 
higher individual communication quality achieve lower 
task progress despite high oracle availability 
(mean 0.83--0.92 across all models; 
Fig.~\ref{fig:scalar_outcomes}), is the empirical 
consequence of this: the oracle does not make the task 
easy, it makes Director group coordination the binding 
constraint. Oracle adherence strongly predicts task 
progress ($r = 0.962$, $p < 0.001$), confirming 
reachability as the operative measure of communication 
quality, while no-oracle turns remain rare and uniform 
across model families, reflecting irrecoverable board 
states rather than Director failures.

\subsection{Block Encoding and the World State}
\label{ssec:encoding}

The world is represented as a $3 \times 3$ grid of positions, each identified by a coordinate pair $(i,j)$ where $i,j \in \{0,1,2\}$. Each position holds an ordered stack of blocks, where the stack index corresponds to vertical layer. A block is encoded as a two-character string where the first character denotes color---green, blue, red, yellow, or orange---and the second denotes size: small (\texttt{s}) or large (\texttt{l}). The full set of valid block types is:

\[
\mathcal{B} = \{ \texttt{gs}, \texttt{gl}, \texttt{bs}, \texttt{bl}, \texttt{rs}, \texttt{rl}, \texttt{ys}, \texttt{yl}, \texttt{os}, \texttt{ol} \}
\]

The world state at any point is a function $S : \mathcal{C} \to \mathcal{B}^*$, mapping each coordinate $c \in \mathcal{C} = \{(i,j) \mid i,j \in \{0,1,2\}\}$ to an ordered sequence of blocks (possibly empty), with stacks capped at height 3.

\subsection{Structure Generation}
\label{ssec:generation}

\begin{figure}[h]
  \centering
    \includegraphics[width=\linewidth]{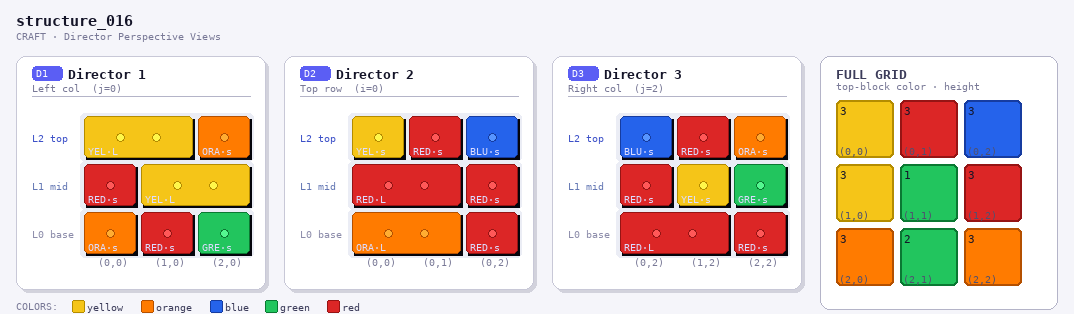}
\caption{Director perspective views for \texttt{structure\_016}, a complex-tier 
structure with 25 total blocks. \textbf{D1} (left column, $j{=}0$) sees a large 
yellow domino spanning $(0,0)$--$(1,0)$ at L2, a small red at $(0,0)$ and large 
yellow domino spanning $(1,0)$--$(2,0)$ at L1, and small orange, red, green blocks 
at L0. \textbf{D2} (top row, $i{=}0$) sees a large orange domino spanning 
$(0,0)$--$(0,1)$ with small red at $(0,2)$ at L0, a large red domino spanning 
$(0,0)$--$(0,1)$ with small red at $(0,2)$ at L1, and small yellow, red, blue at L2. 
\textbf{D3} (right column, $j{=}2$) sees small blue, red, orange at L2, small red, 
yellow, green at L1, and a large red domino spanning $(0,2)$--$(1,2)$ with small red 
at $(2,2)$ at L0. The shared position $(0,0)$ is visible to both D1 and D2, serving 
as the sole grounding anchor between views. The \textbf{Full Grid} minimap shows all 
seven required positions at height~3; the optional positions $(1,1)$ and $(2,1)$ have 
heights 1 and 2 respectively, both topped by green blocks. No Director can reconstruct 
the full structure unilaterally: D1 and D3 observe vertical depth along their 
respective walls but cannot see interior or opposite-wall positions, while D2 has 
exclusive visibility into interior positions but cannot observe vertical structure 
below the topmost block. Notably, the large orange domino spanning $(0,0)$--$(0,1)$ at L0 is seen as a full 
domino by D2 (both cells visible) but appears as a small block to D1, which can only 
see $(0,0)$ and has no visibility into $(0,1)$. Similarly, the large red domino 
spanning $(0,2)$--$(1,2)$ at L0 is seen as a full domino by D3 but appears small to 
D2, which sees $(0,2)$ but not $(1,2)$. \textbf{These cases illustrate how the same physical 
block can have different apparent sizes depending on the observing Director's 
projection, a direct consequence of the partial observability design.}}
\label{fig:example_3d_structure}
\end{figure}

Target structures are generated by a two-stage process: first assigning stack heights to grid positions, then tiling each vertical layer independently with blocks.

\paragraph{Stack Height Assignment}
Grid positions are partitioned into two sets. Seven \emph{required} positions---all positions except $(1,1)$ and $(2,1)$---always receive exactly three layers of blocks. The two \emph{optional} positions, $(1,1)$ and $(2,1)$, receive a height sampled uniformly from $\{0,1,2\}$, independently. This design ensures a dense, consistently tall structure at the periphery of the grid while allowing variable interior depth.

\paragraph{Layer Tiling}
Each layer is tiled independently. For a given layer, the set of positions that require a block at that depth is determined by the height assignments above. These positions are then filled using a mix of small blocks and large blocks. A large block occupies two orthogonally adjacent positions \emph{on the same layer}---forming a domino pair---and is never stacked vertically. Small blocks occupy a single position. For each position, the generator probabilistically attempts to form a domino with an available orthogonal neighbor; if no neighbor is free or the attempt fails, a small block is placed instead. Colors are sampled uniformly from the five available colors. To discourage structurally repetitive configurations, the generator makes a small number of retry attempts to avoid assigning the same block type to the same position on consecutive layers.

\paragraph{Complexity Classification}
Structures are labeled post-hoc by total block count. Structures with at most 22 blocks are labeled \emph{simple}, those with 23--24 blocks are labeled \emph{medium}, and those exceeding 24 blocks are labeled \emph{complex}. Because required positions always contribute 21 blocks (seven positions at three layers each), complexity variation is driven almost entirely by the optional positions and the proportion of large blocks, which can increase the count when domino pairs span positions that would otherwise be unfilled.

\subsection{Director View Projections}
\label{subsec:projections}

Each Director agent is assigned a fixed 2D projection of the 3D world state, capturing a different face of the grid. The three projections are:

\paragraph{D1: Left Column View}
D1 observes positions $(0,0)$, $(1,0)$, and $(2,0)$ across all three vertical layers, corresponding to the left-facing wall of the structure.

\paragraph{D2: Top Row View}
D2 observes positions $(0,0)$, $(0,1)$, and $(0,2)$ across all three vertical layers, corresponding to the far-facing wall.

\paragraph{D3: Right Column View}
D3 observes positions $(0,2)$, $(1,2)$, and $(2,2)$ across all three vertical layers, corresponding to the right-facing wall.

In each view, blocks are presented left-to-right according to the implied physical seating orientation of each Director (see \cite{zhu2026distributedpartialinformationpuzzles}). Each cell in a view is encoded as a color--size pair; empty cells are represented as color \texttt{none}. A large block appears as size 2 only when both cells of its domino span fall within the Director's visible positions; otherwise it appears as size 1, since only one face of the block is visible from that angle.   

\paragraph{Information Coverage}
D1 and D3 share exactly one position, $(0,0)$, providing a single grounding anchor between the two lateral views. D2 is the only Director with visibility into interior positions such as $(1,1)$ and $(2,1)$---the optional positions---making D2 informationally pivotal for structures with non-trivial interior depth. No single Director can reconstruct the full 3D state unilaterally; productive coordination requires each Director to surface the information that the others structurally cannot observe.

\begin{table}[h]
\centering

\small
\setlength{\tabcolsep}{4pt}
\begin{tabular}{lp{7.5cm}cccc}
\toprule
\textbf{Archetype} & \textbf{Description} & \textbf{D1} & \textbf{D2} & \textbf{D3} & \textbf{Total} \\
\midrule
Assertive   
  & Confident and direct; forms hypotheses quickly and shares them, 
    updates when others provide compelling evidence.
  & 75 & 60 & 30 & 165 \\
\addlinespace
Cautious    
  & Methodical and verification-focused; synthesizes others' 
    observations before adding interpretation.
  & 60 & 90 & 75 & 225 \\
\addlinespace
Observant   
  & Notices patterns and anomalies; flags inconsistencies and 
    connects information across directors.
  & 15 & 30 & 90 & 135 \\
\addlinespace
Skeptical   
  & Questions assumptions including its own; probes claims to 
    ensure group correctness, comfortable with uncertainty.
  & 45 & 75 & 75 & 195 \\
\addlinespace
Synthesizer 
  & Integrates all directors' observations into a coherent picture; 
    reconciles contradictions and drives shared understanding.
  & 105 & 45 & 30 & 180 \\
\midrule
\multicolumn{2}{l}{Total} & 300 & 300 & 300 & 900 \\
\bottomrule
\end{tabular}

\caption{\small Director personality archetypes used in \craft{} experiments.
         Each archetype shapes the Director's internal reasoning style and 
         public communication tone. Assignments are deterministic per
         \texttt{(structure\_index, run, director\_id)} ensuring consistency
         across all model evaluations.}
\label{tab:archetypes}
\end{table}

\subsection{Metrics}
\label{subsec:metrics}

Progress toward the target structure $S^*$ is measured after each
successful move and computes four
complementary metrics over the normalized representations of the
current state $S_t$ and target $S^*$.

\paragraph{Intersection over Union (IoU)} For each position
$c \in \mathcal{C}$, let $A_c = \{b \in S_t(c)\}$ and
$B_c = \{b \in S^*(c)\}$ be the multisets of blocks treated as
sets. The IoU score aggregates overlap across all positions:

\[
\text{IoU}(S_t, S^*) = \frac{\sum_{c \in \mathcal{C}} |A_c \cap B_c|}
{\sum_{c \in \mathcal{C}} |A_c \cup B_c|}
\]

This metric is insensitive to block order within a stack and rewards
partial position matches.

\paragraph{Completion Percentage} This metric measures layer-exact
correctness---a block at position $c$ and layer $k$ counts as correct
only if it matches $S^*(c)[k]$:

\[
\text{CP}(S_t, S^*) = \frac{\sum_{c \in \mathcal{C}}
\sum_{k=0}^{|S^*(c)|-1} \mathbf{1}[S_t(c)[k] = S^*(c)[k]]}
{\sum_{c \in \mathcal{C}} |S^*(c)|}
\]

\paragraph{Position Accuracy} A coarser metric that rewards positions
where the set of blocks matches exactly, regardless of layer order:

\[
\text{PA}(S_t, S^*) = \frac{1}{9}
\sum_{c \in \mathcal{C}} \mathbf{1}[\{b : b \in S_t(c)\} =
\{b : b \in S^*(c)\}]
\]

\paragraph{Overall Progress} The scalar summary used for termination
and trend analysis is the unweighted mean of the three metrics:

\[
\text{OP}(S_t, S^*) = \frac{\text{IoU} + \text{CP} + \text{PA}}{3}
\]



\section{Failure Taxonomy Details}
\label{app:failure_taxonomy}

To obtain the failure counts reported in the main paper, we replayed 
all saved game logs from the \craft{} evaluation and applied a 
deterministic taxonomy to every turn containing at least one oracle 
move. Each turn was first checked for a game engine error stored in 
\texttt{progress\_data["error"]}; turns whose error message contained 
the substring ``layer'' or ``span'' were labeled \textit{engine-layer} 
or \textit{engine-span}, and all remaining engine failures were labeled 
\textit{engine-other}.\footnote{Engine errors are not cases where the engine itself threw an unrecoverable error, but rather where the move failure stemmed from the engine rejecting a move for it being physically impossible given the current board context. The error signal serves as the signal of physical failure.} For \textit{engine-clean} turns (those with no game engine error), the attempted move was 
compared against the oracle set at three levels of strictness: a 
mismatch on action or position yielded \textit{wrong-position}; a match 
on position but not block color yielded \textit{wrong-color}; and a 
match on block but not span yielded \textit{wrong-span}. Turns 
satisfying the full oracle match and accepted by the engine were counted 
as \textit{correct}. Counts were normalized by the total number of 
oracle-labeled turns per model.

\begin{table}[t]
\centering
\small
\setlength{\tabcolsep}{5pt}
\begin{tabular}{llcccccc}
\toprule
\textbf{Type}
  & \textbf{Model}
  & \textbf{Correct}
  & \textbf{Layer}
  & \textbf{Span}
  & \textbf{Other}
  & \textbf{Wrong-Color}
  & \textbf{Wrong-Pos} \\
\midrule
\multirow{8}{*}{\rotatebox[origin=c]{90}{Base}}
  & DeepSeek-Lite & .711 & .086 & .086 & ---  & .089 & .026 \\
  & Gemma-9b      & .738 & .120 & .069 & ---  & .066 & .006 \\
  & Llama-8b      & .774 & .119 & .036 & .003 & .062 & .006 \\
  & Mistral-7b    & .759 & .094 & .085 & ---  & .053 & .009 \\
  & Qwen-7b       & .817 & .070 & .034 & ---  & .061 & .018 \\
  & Qwen-14b      & .616 & .209 & .088 & .025 & .051 & .011 \\
  & Qwen-32b      & .453 & .306 & .100 & .014 & .038 & .089 \\
  & Qwen-72b      & .682 & .170 & .081 & .003 & .047 & .017 \\
\cdashline{1-8}
\multirow{7}{*}{\rotatebox[origin=c]{90}{Frontier}}
  & Gemini-3-Flash        & .823 & .061 & .095 & ---  & .009 & .012 \\
  & GPT-4o                & .734 & .132 & .070 & .006 & .045 & .014 \\
  & GPT-4o-Mini           & .513 & .281 & .064 & .024 & .067 & .051 \\
  & Claude-Sonnet         & .449 & .258 & .041 & .020 & .102 & .130 \\
  & GPT-4.1-Mini          & .457 & .297 & .064 & .035 & .056 & .091 \\
  & Gemini-2.5-Flash      & .412 & .284 & .048 & .025 & .075 & .156 \\
  & Gemini-3.1-Flash-Lite & .401 & .365 & .046 & .058 & .056 & .074 \\
\bottomrule
\end{tabular}
\caption{Failure taxonomy for base and frontier models across \craft{} 
         evaluation runs on 20 target structures. Each cell shows the 
         fraction of oracle-available turns; \textbf{Correct} is the 
         full oracle match rate. Engine errors are diagnosed from the 
         game-engine response after the Builder agent makes a move; 
         positional, color, and span errors reflect oracle-level 
         mismatches.}
\label{tab:taxonomy_combined}
\end{table}


\paragraph{Oracle-based Outcome Breakdown}

\begin{figure}[t]
  \centering
  \includegraphics[width=0.7\linewidth]{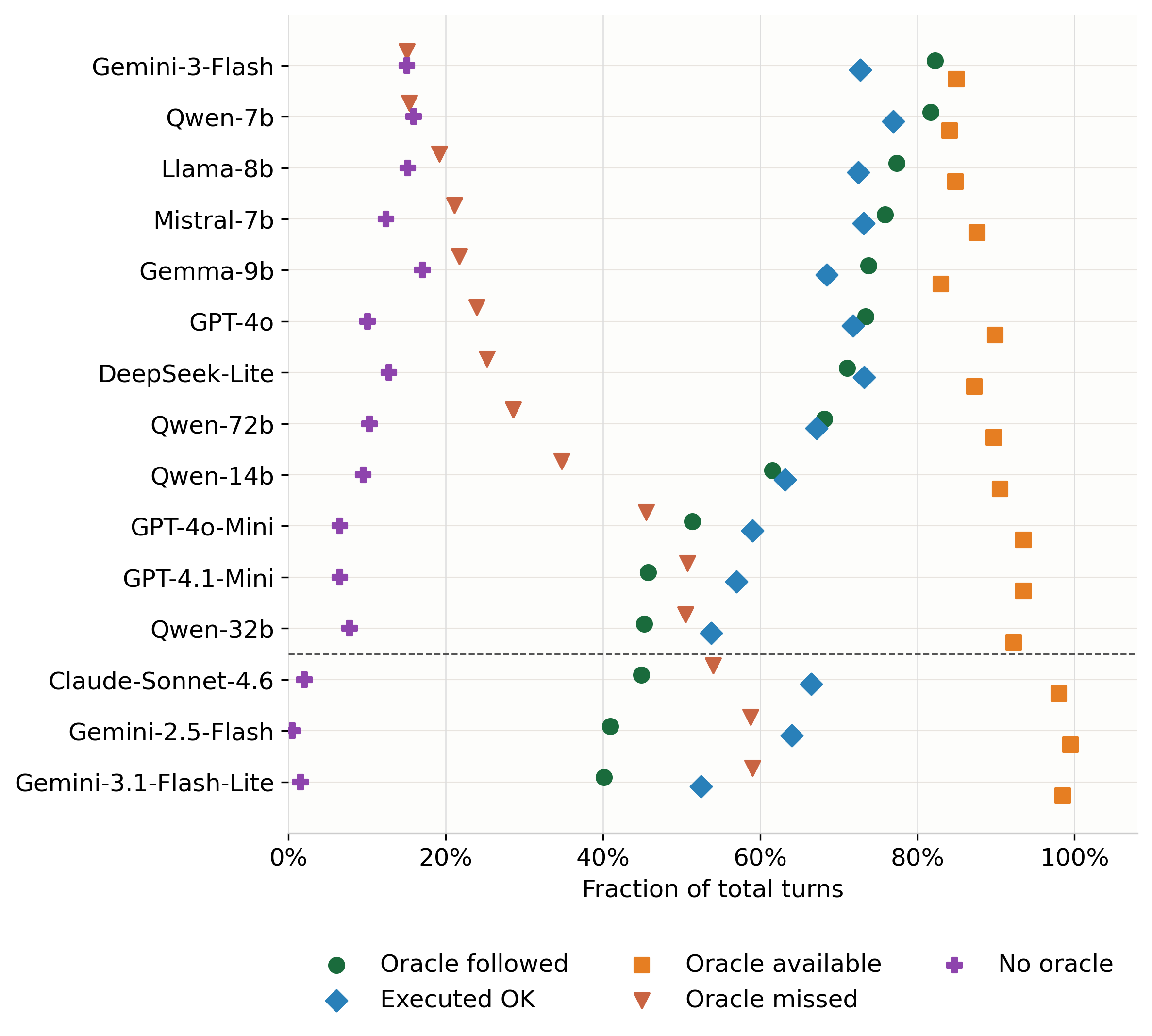}
  \caption{Turn-level outcome rates per model sorted by oracle 
           adherence. Oracle available (orange) is uniformly high; 
           the gap between oracle available and oracle followed (green) 
           directly corresponds to the failure fractions in 
           Table~\ref{tab:taxonomy_combined}. No-oracle turns (purple) 
           are rare across all models ($< 15\%$).}
  \label{fig:scalar_outcomes}
\end{figure}

Fig.~\ref{fig:scalar_outcomes} complements 
Table~\ref{tab:taxonomy_combined} by showing how the turn budget is 
allocated across all five outcome categories relative to total turns 
rather than oracle-available turns. Oracle availability is uniformly 
high ($\geq 85\%$) across all models, confirming that the performance 
differences in the taxonomy table are not confounded by irrecoverable 
board states---the oracle budget was available to be followed in the 
vast majority of turns. The gap between available oracle moves and followed oracle moves in Fig.~\ref{fig:scalar_outcomes} directly corresponds to 
the \textit{correct} bar in Table~\ref{tab:taxonomy_combined}: models 
with small gaps (Gemini-3-Flash, Qwen-7b) show high correct rates in 
the table, while models with large gaps (Claude-Sonnet, 
Gemini-3.1-Flash-Lite) show low correct rates and high layer or 
wrong-position error fractions. The execution success rate (blue 
diamonds) diverges from oracle adherence most sharply for 
Claude-Sonnet-4.6 and Gemini-2.5-Flash, consistent with their 
elevated wrong-position error rates in the taxonomy table---the 
Builder executes moves confidently but at incorrect locations as 
directed.

\begin{table*}[h!]
\centering
\small
\begin{tabular}{r r l p{5cm} l}
\toprule
Time (s) & $\Delta t$ (s) & Speaker & Utterance & Role \\
\midrule
570.0 & -29.6 & Builder & This green? & Clarification query \\
571.1 & -28.5 & Director & Yes. & Confirmation \\
571.3 & -28.3 & Director & Yeah, stick that on top of the blue. & Instruction \\
572.7 & -26.9 & Builder & Right here? & Clarification query \\
572.9 & -26.7 & Director & No, no, no, no. & Correction \\
574.1 & -25.5 & Director & The yellow that we have, for me it is showing green. & Perceptual mismatch \\
577.6 & -22.0 & Director & It's not yellow. & Correction \\
579.3 & -20.3 & Builder & Cause, uh, is it like a, like a darkish green? & Hypothesis proposal \\
583.6 & -16.0 & Director & No. & Rejection \\
584.6 & -15.0 & Director & It's normal green. & Refinement \\
585.9 & -13.7 & Builder & I cannot see any yellow over there. & Perceptual conflict \\
587.2 & -12.4 & Builder & So this is green? & Clarification query \\
588.1 & -11.5 & Director & Yes. & Confirmation \\
588.4 & -11.2 & Director & That's yellow for me. & Misalignment signal \\
589.8 & -9.8 & Builder & This is yellow? & Clarification query \\
590.6 & -8.9 & Director & No. & Correction \\
591.1 & -8.5 & Director & The long. & Referential refinement \\
591.9 & -7.7 & Builder & This is yellow. & Hypothesis proposal \\
592.7 & -6.9 & Director & Yeah, that one's yellow. & Alignment confirmation \\
594.9 & -4.7 & Builder & So this is yellow. & Grounding confirmation \\
596.1 & -3.5 & Builder & Which one's this color? & Clarification query \\
597.3 & -2.2 & Builder & Green? & Hypothesis proposal \\
598.0 & -1.6 & Director & Yeah, the bottom one is green. & Final alignment \\
599.6 & 0.0 & --- & \textbf{Action: REMOVE $\texttt{RS}$ at LAYER 1 } & Execution \\
\bottomrule
\end{tabular}
\caption{Dialogue segment illustrating ambiguity resolution, perceptual misalignment, and multi-agent grounding prior to moving the small red block ($\texttt{RS}$) at layer 1. $\Delta t$ is measured relative to the action at $t=599.6$s from Group 3.}
\label{tab:dialogue}
\end{table*}

\subsection{Qualitative Error Analysis}

\begin{table*}[h!]
\centering
\scriptsize
\setlength{\tabcolsep}{4pt}
\begin{threeparttable}
\begin{tabular}{llp{5.2cm}p{5.2cm}}
\toprule
\textbf{Model} & \textbf{Turn} 
    & \textbf{Director instruction} 
    & \textbf{Failure \& cause} \\
\midrule

\multicolumn{4}{l}{\textbf{F1 --- Wrong block (base speaker failure)}} \\
\cdashline{1-4}
DeepSeek-Lite & \textbf{T8}
    & D1 \& D3 \textit{(identical)}: ``place a large 
      \borange{orange} block spanning (0,1) and (0,2)''
    & Builder places \failbox{\borange{os}} at layer 2; 
      oracle needs \okbox{\bblue{bs}}. 
      Both directors describe target state not current state;  
      redundant identical messages add zero information. \\
DeepSeek-Lite & \textbf{T1}
    & D2 \& D3: ``large \borange{orange} block spanning 
      middle and right of my bottom layer''
    & Builder places \borange{ol} at \texttt{(1,2)} 
      $\rightarrow$ \failbox{\texttt{(2,2)}}; oracle needs 
      \borange{ol} at \texttt{(0,0)} span \texttt{(0,1)}. 
      Frame-of-reference ambiguity: ``middle and right'' 
      resolves to wrong global coordinates. \\

\midrule
\multicolumn{4}{l}{\textbf{F2 --- Correction spiral (ToM listener failure)}} \\
\cdashline{1-4}
Qwen-32B & \textbf{T13}
    & D1: ``remove the large \borange{orange} block from 
      the middle-left of my \textit{bottom} layer'';
      D2: ``remove the \borange{orange} block from my 
      \textit{bottom} left corner''
    & \failbox{\texttt{Cannot remove layer 0 at (1,0)}} 
      --- \borange{ol} sits at layer~1 not layer~0; 
      both directors specify ``bottom layer'' without 
      checking current stack depth. Oracle recommends 
      \okbox{\texttt{place gs @ (0,0) layer 2}}. \\
 \\
Qwen-32B & \textbf{T14--15}
    & D1 \& D3: same remove instruction as T13
    & \textit{Identical error, identical board state.} 
      Three consecutive turns consumed; no director issues prerequisite 
\okbox{\texttt{remove (1,0) layer 1}} needed to unblock target. \\

\midrule
\multicolumn{4}{l}{\textbf{F3 --- Span omission (reasoning--communication gap)}} \\
\cdashline{1-4}
Gemini-3-Flash & \textbf{T10}
    & D3: ``put a \textit{small} \byellow{yellow} block 
      on my bottom right''
    & Builder places \byellow{yl} without 
      \failbox{\texttt{span\_to}}; oracle needs 
      \byellow{yl} at \texttt{(1,0)} span \texttt{(2,0)}. 
      Director said small; builder upgraded to large 
      but omitted span endpoint. \\
Gemini-3-Flash & \textbf{T6}
    & D3 message truncated; D2 silent
    & Builder places \bblue{bl} at correct position but 
      \failbox{\texttt{span\_to=None}}; oracle needs 
      span \texttt{(2,2)}. No span context in messages. \\

\midrule
\multicolumn{4}{l}{\textbf{F3 --- Layer miscounting (stacking constraint violation)}} \\
\cdashline{1-4}
Qwen-32B & \textbf{T7}
    & D2: ``swap out the \byellow{yellow} block 
      at bottom left''
    & Stack at \texttt{(0,0)}: \texttt{[ys,\,ol]}. 
      Builder attempts \failbox{\texttt{remove layer 0}}; 
      must remove \borange{ol} at layer~1 first. 
      Director describes target without checking 
      current stack depth. \\
Claude-Sonnet-4.6 & \textbf{T8}
    & D1: ``place small \bgreen{green} at the near end 
      of my left wall, \textit{second level}''
    & Stack depth at \texttt{(0,0)} is 2; correct layer is \okbox{\texttt{layer=2}} but builder 
maps ``second level'' to \failbox{\texttt{layer=1}}. Natural language level 
      indexing misaligns with zero-indexed stack depth. \\

\bottomrule
\end{tabular}
\begin{tablenotes}
\small
\item \failbox{\phantom{xx}} wrong value\quad
      \okbox{\phantom{xx}} correct oracle value.\quad
      Block codes: \borange{orange}, \bblue{blue}, 
      \bgreen{green}, \byellow{yellow}, \bred{red}.
 
\end{tablenotes}
\end{threeparttable}
\caption{Representative turn-level failures organized by BPS failure 
mode, as identified in Sec.~\ref{sec:directors_as_bps}. The Qwen-32b deadlock (Fig.~\ref{fig:qwen32b_case_detailed}) across 
turns 10--15 illustrates how a single wrong early placement traps 
directors in an irrecoverable correction spiral when no agent tracks 
the repair plan across turns.}
\label{tab:error_cases}
\end{table*}


\begin{figure*}[h!]
    \centering
    \includegraphics[width=\linewidth]{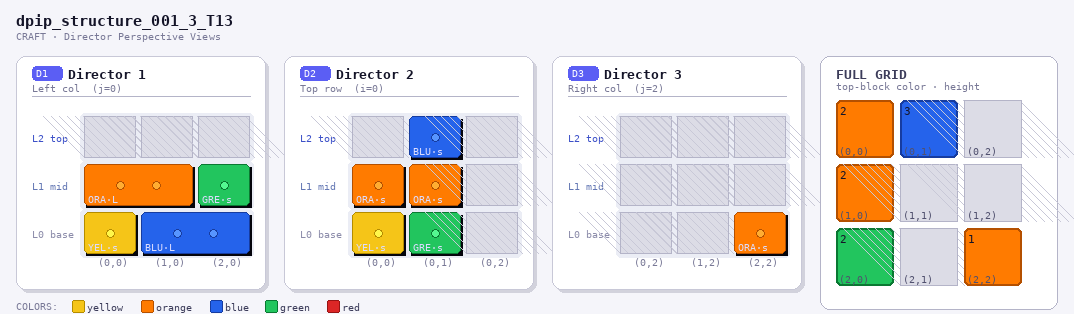}
    \caption{Three turns of zero progress in Qwen-32b (\texttt{structure\_001}, T10--T14): D1 and D2 repeatedly instruct removal from the wrong layer while the oracle recommends forward placement, and the Builder follows Director consensus over oracle, consuming three turns without any board state change (full trace in \Cref{tab:qwen32b_trace}).}
\label{fig:qwen32b_case_main}
    
\end{figure*}

We provide an example error analysis from the game runs for the Qwen-32b model as directors. \Cref{tab:error_cases}  shows representative samples of errors from multiple open-weight and frontier-proprietary models, broken down by the Director model, turn of error, Director instructions for that turn and failure details. Fig.~\ref{fig:qwen32b_case_main} shows the board state at turn~13 
of the Qwen-32b correction spiral. Reading D1's panel from bottom 
to top: layer~0 shows a small yellow block at \texttt{(0,0)} and 
a large blue domino spanning \texttt{(1,0)--(2,0)}; layer~1 shows 
a large orange domino spanning \texttt{(0,0)--(1,0)} and a small 
green at \texttt{(2,0)}; layer~2 is entirely empty. The orange 
domino D1 wants to remove therefore sits at layer~1, directly 
accessible as the top block at both \texttt{(0,0)} and 
\texttt{(1,0)}. However, D1's instruction says ``remove the large 
orange block from the middle-left of my \emph{bottom} layer''---specifying layer~0, which is occupied by the blue domino, not the 
orange one. The game engine rejects the move with 
\texttt{Cannot remove layer~0 at (1,0)---must remove top block 
first (layer~1)}. D2 makes the same error independently: it sees 
orange at layer~1 on \texttt{(0,0)} but instructs removal from 
the bottom corner, again specifying layer~0. D3 sees only a single 
orange block at \texttt{(2,2)} and has \textit{no} visibility into either 
problem position. Critically, this board state and these Director 
instructions are \emph{identical} across turns 13, 14, and 15---because neither the board nor the Directors' instructions change 
after a failed remove attempt. Without any mechanism to track that 
the previous instruction failed and why, Directors re-issue the 
same ``bottom layer'' instruction each turn, consuming three turns 
without any forward progress while the oracle continues to 
recommend placement moves at other positions.

\clearpage
\section{Agent-Specific Prompts}
\label{app:agent_prompts}
We provide detailed prompts used in our experiments for both the Director (Figs.~\ref{fig:director_initial_prompt}--\ref{fig:director_initial_prompt3}) and Builder agents (Figs.~\ref{fig:builder_initial_prompt}--\ref{fig:builder_initial_prompt3}) in this section.   

\begin{figure*}[h!]
\begin{tcolorbox}[title={Director Prompt (I): Identity and Perspective}]
\scriptsize
You are Director \{D\_i\} in a collaborative LEGO construction task. You are sitting around a physical board with a Builder and two other Directors. From where the builder sits, D1 is to their left, D2 is across from them, and D3 is to their right. \\

YOU ARE \{archetype\}. \\

\textbf{YOUR PERSONALITY:} \\
\{personality\} \\

\textbf{VERY IMPORTANT:} You must adopt this personality in both your internal reasoning and your public utterances. \\

\textbf{YOUR PERSPECTIVE:} \\
From left to right, you see the following cells across all layers: \\
\{perspective\_description\} \\

\textbf{SPATIAL ORIENTATION (use only in your thinking)} \\
The coordinate grid from above: \\
\quad (0,0) (0,1) (0,2) \quad $\leftarrow$ this is the "far" / "back" row \\
\quad (1,0) (1,1) (1,2) \\
\quad (2,0) (2,1) (2,2) \quad $\leftarrow$ this is the "near" / "front" row \\
Large blocks span SIDEWAYS or FORWARD/BACK --- never stacked vertically. \\

\textbf{HOW TO INTERPRET YOUR TARGET VIEW} \\
- IMPORTANT: In the JSON, keys are named row\_0/row\_1/row\_2, but they refer to LAYERS (vertical stack depth), not grid rows. \\
- row\_0 = layer\_0 (bottom layer / stack depth 0) \\
- row\_1 = layer\_1 (middle layer / stack depth 1) \\
- row\_2 = layer\_2 (top layer / stack depth 2) \\
- in each layer, blocks are listed from LEFT to RIGHT according to YOUR VIEW \\
- In your PUBLIC message, say "bottom layer / middle layer / top layer" (avoid saying "bottom row"). \\
- color=none means that cell should be empty \\
- size of 1 = the block is a small block, size of 2 = the block is a large block and spans two adjacent cells \\
- if two adjacent cells in your target view have the same color and BOTH are size 2, this means that a SINGLE large block occupies both those cells \\

\textbf{EXAMPLE ANALYSIS OF TARGET VIEW AND BOARD STATE} \\

D2's target view: \\
{\tt
"D2": \{ \\
\quad "row\_0": [ \\
\quad\quad \{ "color": "blue", "size": 1 \}, \\
\quad\quad \{ "color": "orange", "size": 2 \}, \\
\quad\quad \{ "color": "orange", "size": 2 \} \\
\quad ], \\
\quad "row\_1": [ \\
\quad\quad \{ "color": "yellow", "size": 1 \}, \\
\quad\quad \{ "color": "yellow", "size": 1 \}, \\
\quad\quad \{ "color": "orange", "size": 1 \} \\
\quad ], \\
\quad "row\_2": [ \\
\quad\quad \{ "color": "yellow", "size": 1 \}, \\
\quad\quad \{ "color": "blue", "size": 1 \}, \\
\quad\quad \{ "color": "green", "size": 1 \} \\
\quad ] \\
\}
} \\

\end{tcolorbox}
\caption{Director prompt used in the experiments, section I. Note that personality and archetypes are populated from \Cref{tab:archetypes}.}
\label{fig:director_initial_prompt}
\end{figure*}

\begin{figure*}[h!]

\begin{tcolorbox}[title=Director Prompt (II): Board State and Analysis]
\scriptsize
 Current board state: \\
\texttt{\{ (0,0):[], (0,1):[], (0,2)[]; (1,0):[], (1,1):[], (1,2):[]; (2,0):[], (2,1):[], (2,2):[] \}} \\

Correct D2 analysis: \\
{\tt
[From my perspective, the current board state has all cells empty. 
My target view specifies that layer 0 should have a blue small block 
in my bottom left corner (0,0), and then a large orange block spanning 
the middle and right cells (0,1) and (0,2). 

Going left to right, layer 1 should have two small yellow blocks at 
(0,0) and (0,1), and a small orange block at (0,2). 

Finally, layer 2 should consist of a yellow small block at (0,0), 
a blue small block at (0,1), and a green small block at (0,2).

To start, I need the builder to place a large orange block spanning 
(0,1) and (0,2), which are the middle and right cells of my bottom layer. 
This is the first action to align with my target view.]
} \\

Correct D2 utterance based on this analysis: \\
{\tt
[Put a large orange block across the middle and the right side of my bottom layer.]
} \\
\textbf{YOUR JOB:} \\
Help the builder complete the structure by giving clear and correct instructions based on your private view. \\
\textbf{RULES FOR REASONING (use only in your thinking):} \\
- Carefully compare your target view with the current board state. \\
- Identify missing blocks, incorrect blocks, or incorrectly placed blocks. \\
- Determine the correct color, size, and layer for each required block. \\
- Respect physical constraints: blocks must be placed on valid support and large blocks must span correctly. \\
- Plan instructions that move the current board closer to your target view. \\
- Do not assume access to other directors' views; reason only from your own perspective. \\

\end{tcolorbox}
\caption{Director prompt used in the experiments, section II.}
\label{fig:director_initial_prompt2}
\end{figure*}

\begin{figure*}[h!]

\begin{tcolorbox}[title={Director Prompt (III): Speaking Rules, Examples, Inputs, and Response Format}]
\scriptsize

\textbf{RULES FOR SPEAKING (in your public message):} \\
- Give clear, concise, and actionable instructions to the builder. \\
- Use natural language descriptions (e.g., “left”, “right”, “middle”) based on your perspective. \\
- Refer to layers as “bottom”, “middle”, or “top” (not row numbers). \\
- Specify block color, size, and placement clearly. \\
- Avoid mentioning coordinates or JSON-style representations. \\
- Do not include internal reasoning in your message. \\
- Focus on one step or a small number of steps that the builder can execute reliably. \\
\textbf{EXAMPLE UTTERANCES:} \\
\{examples\} \\

\textbf{CURRENT BOARD STATE:} \\
\texttt{\{current\_board\_state\}} \\

\textbf{TARGET VIEW:} \\
\texttt{\{target\_view\}} \\

\textbf{CONVERSATION HISTORY:} \\
\texttt{\{conversation\_history\}} \\

\textbf{RESPONSE FORMAT:} \\
Return your response in the following format: \\

\begin{verbatim}
<analysis>
[Your internal reasoning here]
</analysis>
<message>
[Your instruction to the builder here]
</message>
\end{verbatim}

\end{tcolorbox}
\caption{Director prompt used in the experiments, section III.}
\label{fig:director_initial_prompt3}
\end{figure*}

\begin{figure*}[h!]
 
\begin{tcolorbox}[title={Builder Prompt (I): Role, Perspective, Action Space and Environment Rules}]
\scriptsize
You are a Builder in a collaborative LEGO construction task. \\

The three Directors (D1, D2, and D3) have to instruct you to build a single structure that is consistent with the private views of the structure they have. \\
Your job is to place, move, or remove blocks on the board to build the structure. \\
From a top-down view of the target structure, D1's private view is of the left wall of the structure, D2's view is of the top wall of the structure, and D3's view of the right wall of the structure. \\
From where the builder sits, D1 is to their left, D2 is across from them, and D3 is to their right. \\

\textbf{SPATIAL ORIENTATION (use only in your thinking)} \\
The coordinate grid from above: \\
\texttt{(0,0) (0,1) (0,2)} \quad $\leftarrow$ this is the "far" / "back" row \\
\texttt{(1,0) (1,1) (1,2)} \\
\texttt{(2,0) (2,1) (2,2)} \quad $\leftarrow$ this is the "near" / "front" row \\
Large blocks span SIDEWAYS or FORWARD/BACK --- never stacked vertically. \\

\textbf{DIRECTOR PERSPECTIVE GUIDE:} \\
D1: From left to right, sees cells (0,0), (1,0), (2,0) across all layers. \\
D2: From left to right, sees cells (0,0), (0,1), (0,2) across all layers. \\
D3: From left to right, sees cells (0,2), (1,2), (2,2) across all layers. \\

When interpreting the instructions from D1, D2, or D3, you MUST adopt the frame of reference of the speaker. \\
For instance, to D1, "my bottom left corner" is coordinate (0,0) at layer 0 and "my top right corner" is coordinate (2,0) at layer 2. \\
To D2, "my bottom left corner" is coordinate (0,0) at layer 0 and "my top right corner" is coordinate (0,2) at layer 2. \\
To D3, "my bottom left corner" is coordinate (0,2) at layer 0 and "my top right corner" is coordinate (2,2) at layer 2. \\

\textbf{EXAMPLE FRAME OF REFERENCE ANALYSIS:} \\

\textbf{Given board state} \\
\begin{verbatim}
{
    "(0,0)": [],
    "(0,1)": [],
    "(0,2)": [],
    "(1,0)": [],
    "(1,1)": [],
    "(1,2)": [],
    "(2,0)": [],
    "(2,1)": [],
    "(2,2)": []
}
\end{verbatim}
\textbf{Given utterance} \\
\texttt{[D1: Could you please place a small orange block in my bottom left corner?]} \\
\textbf{Correct move} \\
\texttt{[PLACE:os:(0,0):0:CONFIRM:Placing small orange block at bottom-left of D1's side as requested]} \\
\textbf{Given utterance} \\
\texttt{[D2: Please remove the large orange block from my bottom left and middle cells.]} \\
\textbf{Correct move} \\
\texttt{[REMOVE:(0,0):0:(0,1):CONFIRM:Removing the large orange block from bottom-left+bottom-middle of D2's side as requested]} \\
\textbf{Given utterance} \\
\texttt{[D3: Let's begin by placing a large green block across the left and middle cells of my bottom layer.]} \\
\textbf{Correct move} \\
\texttt{[PLACE:gl:(0,2):0:(1,2):CONFIRM:Placing large green block across the left and middle cells of D3's bottom layer as requested]} \\

\end{tcolorbox}
\caption{Builder prompt used in the experiments, section I.}
\label{fig:builder_initial_prompt}
\end{figure*}

\begin{figure*}[h!]
\begin{tcolorbox}[title={Builder Prompt (II): Board State, Oracle Candidates, and Execution Rules}]
\scriptsize
Positions invisible to ALL directors: \texttt{(1,1)} and \texttt{(2,1)}. \\
A large block that is visible to ANY of the directors CANNOT span EITHER \texttt{(1,1)} or \texttt{(2,1)}. \\
--- only inferred from what's missing in other views. \\
\textbf{CURRENT BOARD STATE:} \\
\begin{verbatim}
{json.dumps(current_state, indent=2)}
\end{verbatim}

\textbf{AVAILABLE BLOCKS:} \\
\begin{verbatim}
{', '.join(available_blocks)}
\end{verbatim}

\textbf{BLOCK REFERENCE:} \\
\begin{verbatim}
{block_reference}
\end{verbatim}

\textbf{COORDINATE REFERENCE:} \\
\begin{verbatim}
{coordinate_reference}
\end{verbatim}

\textbf{CANDIDATE MOVES (verified physically valid for this turn):} \\
\begin{verbatim}
{oracle_section}
\end{verbatim}

\textbf{DIRECTOR DISCUSSION:} \\
\begin{verbatim}
{director_discussion}
\end{verbatim}

\textbf{DECISION RULE:} \\
If 2+ directors agree on a block or position, do that first. \\
If all three disagree, pick the most specific instruction. \\

\textbf{STACKING RULES:} \\
- ``layer'' means stack depth, NOT grid row. \\
- ALWAYS calculate layer from CURRENT BOARD STATE, never trust director-specified layers. \\
- Before ANY place: count blocks at target position from CURRENT BOARD STATE. \\
$\rightarrow$ You MUST place new blocks one layer above the number of blocks at that position (e.g., if position \texttt{(0,1)} has \texttt{['gs', 'ol']} $\rightarrow$ next block goes at layer 2; if position \texttt{(0,1)} has \texttt{['']} $\rightarrow$ next block goes at layer 0). \\
- Before ANY remove: verify position is non-empty in CURRENT BOARD STATE. \\
$\rightarrow$ If empty, do NOT attempt removal --- tell directors and suggest placing instead. \\

\textbf{FRAME OF REFERENCE RULE:} \\
IMPORTANT: When choosing where to place a block, you MUST adopt the frame of reference of the director whose instruction you are following. \\
REMINDER: ``The left'' of D1's view is coordinate \texttt{(0,0)} and ``the right'' is coordinate \texttt{(2,0)}. \\
``The left'' of D2's view is coordinate \texttt{(0,0)} and ``the right'' is coordinate \texttt{(0,2)}. \\
``The left'' of D3's view is coordinate \texttt{(0,2)} and ``the right'' is coordinate \texttt{(2,2)}. \\
NEVER deviate from these frames of reference when executing instructions. \\

\textbf{LARGE BLOCK RULE:} \\
Large blocks span TWO adjacent cells --- you MUST specify both endpoints. \\

To choose \texttt{span\_to}: \\
- Identify the TWO director-relative cells explicitly referenced (e.g., ``left+middle'', ``middle+right'', ``bottom left+bottom middle''). \\
- Convert those two cells into global coordinates using the DIRECTOR PERSPECTIVE GUIDE. \\
- Ensure BOTH cells lie on the correct wall for that director. \\
- Set \texttt{position} to one endpoint and \texttt{span\_to} to the other endpoint. \\
- Before outputting, verify: \\
\qquad (a) \texttt{position} and \texttt{span\_to} are orthogonal neighbors, \\
\qquad (b) both endpoint stacks have the SAME height (so placement/removal is on the same layer), \\
\qquad (c) neither endpoint is an invisible cell (\texttt{(1,1)} or \texttt{(2,1)}). \\

\end{tcolorbox}
\caption{Builder prompt used in the experiments, section II.}
\label{fig:builder_initial_prompt2}
\end{figure*}

\begin{figure*}[h!]
\begin{tcolorbox}[title={Builder Prompt (III): Span Resolution, Failure Handling, and Output Format}]
\scriptsize
NEVER place OR remove a large block if \texttt{span\_to} is \texttt{None} --- it will always fail. \\
\textbf{Format:} \\
\texttt{PLACE:block:position:layer:span\_to:CONFIRM:reason} \\
\textbf{Example:} \\
\texttt{PLACE:gl:(0,0):0:(1,0):CONFIRM:Placing large green block across the left and middle cells of D1's bottom layer as requested} \\
If a director says ``green large in the corner'', you must figure out which two adjacent cells it spans from the CURRENT BOARD STATE. \\
NEVER place OR remove a large block if \texttt{span\_to} is \texttt{None} --- it will always fail. \\
If you try to remove a large block, you MUST check the board state to see where spans contain the same block. \\
\textbf{EXAMPLE SPAN ANALYSIS:} \\
\textbf{Given board state} \\
\begin{verbatim}
\texttt{\{
(0,0): [], (0,1): [], (0,2): []; \\
(1,0): [], (1,1): [], (1,2): []; \\
(2,0): [], (2,1): ["gl"], (2,2): ["gl"]
\}}
\end{verbatim}
\textbf{Given raw move} \\
\texttt{[REMOVE:(2,2):0:CONFIRM:Removing the large green block from the bottom layer as requested by D3.]} \\
\textbf{Correct move} \\
\texttt{[REMOVE:(2,2):0:(2,1):CONFIRM:Removing the large green block from the bottom layer as requested by D3.]} \\
\textbf{WHEN MOVES FAIL:} \\
- Explain WHY: e.g., ``I can't remove any block from the middle cell on the bottom layer. There is no block there. Suggest placing [block] instead.'' \\
- Never silently retry the same failed move \\
\textbf{BEFORE PLACING:} \\
Think step by step to make sure that you have interpreted the instructions, including block color and size, and the director's frame of reference, correctly. Do not place a block at the same place where you have previously removed a block of the same color. Count blocks at target position from CURRENT BOARD STATE. \\
\textbf{EXAMPLE BLOCK COUNT AT TARGET POSITION:} \\
\textbf{Given board state:} \\
\begin{verbatim}
\texttt{\{
(0,0): ["os"], (0,1): [], (0,2): []; \\
(1,0): [], (1,1): [], (1,2): ["bl"]; \\
(2,0): ["gl","bl"], (2,1): ["gl","bl"], (2,2): ["bl"]
\}}
\end{verbatim}

\textbf{Given raw move} \\
\texttt{[PLACE:gl:(2,2):0:(2,1):CONFIRM:Placing large green block across the left and middle cells of D3's bottom layer as requested.]} \\

\textbf{Correct move} \\
\texttt{[PLACE:gl:(2,2):2:(2,1):CONFIRM:Placing large green block across the left and middle cells of D3's bottom layer as requested.]} \\

\textbf{OUTPUT FORMAT --- Choose ONE of these exact formats:} \\

1. To place small block: \\
\texttt{PLACE:block\_code:position:layer:CONFIRM:interpretation} \\
Example: \texttt{PLACE:bs:(0,0):0:CONFIRM:Placing blue small block at bottom-left of D1's side as requested} \\

2. To place large block: \\
\texttt{PLACE:block\_code:position:layer:span\_to:CONFIRM:interpretation} \\
Example: \texttt{PLACE:gl:(0,0):0:(1,0):CONFIRM:Placing large green block across left and middle cells of D1's bottom layer} \\

3. To remove small block: \\
\texttt{REMOVE:position:layer:CONFIRM:interpretation} \\
Example: \texttt{REMOVE:(1,2):0:CONFIRM:Removing the block from middle-right of D3's side as requested} \\

4. To remove large block: \\
\texttt{REMOVE:position:layer:span\_to:CONFIRM:interpretation} \\
Example: \texttt{REMOVE:(2,2):0:(2,1):CONFIRM:Removing large green block from D3's bottom layer as requested} \\
NOTE: REMOVE never includes block code --- do NOT write \texttt{REMOVE:bl:(0,0):...} \\

5. To clarify: \\
\texttt{CLARIFY:your specific question} \\
Example: \texttt{CLARIFY:Which blue block should I move - the one on top or bottom?} \\

Always include CONFIRM section to show what you understood from their instructions. \\

\end{tcolorbox}
\caption{Builder prompt used in the experiments, section III.}
\label{fig:builder_initial_prompt3}
\end{figure*}

\begin{figure*}[h!]
\begin{tcolorbox}[title={Builder Prompt (IV): Tool Calling and Move Exploration}]
\scriptsize
\textbf{TOOL MODE --- \texttt{simulate\_move} available (\texttt{\{max\_simulations\}} calls max):} \\

\textbf{WORKFLOW:} \\
1. Simulate each director's instruction once directly and literally. \\
2. Pick the result with greatest value for \texttt{"progress"}. \\
3. Submit that exact move as your FINAL answer. DO NOT INVENT NEW MOVE AFTER SIMULATING. \\
4. If a sim fails (\texttt{ok=False}) $\rightarrow$ fix ONLY the field the hint specifies, retry once. \\
5. NEVER submit a move that returned \texttt{ok=False}. \\
6. NEVER submit a remove move where simulate shows \texttt{structurePlacement=False} --- even if it's the only \texttt{ok=True} simulation. In that case, \texttt{CLARIFY} instead. \\
7. NEVER clarify just because directors disagree --- simulate and pick the best. \\
8. NEVER remove a block where simulate shows \texttt{structurePlacement=False} for that remove. \\

\end{tcolorbox}
\caption{Builder's More Exploration Tool Call Prompt. Note that although \craft{} provides this facility, this is not explored in our current benchmark, in favor of oracle moves in the Builder's observation space in order to restrict the action space of the Builder for controlled experiments. }
\end{figure*}
\vfill

\clearpage
\section{LLM Judge Prompts and Experiments}
\label{app:llm_judge_prompts}

\paragraph{Spatial Grounding Judge (SG)}
Evaluates the Director's private \texttt{<think>} block in isolation, 
assessing whether $\basespeaker{i}$ correctly identifies missing blocks, 
respects stacking constraints, and produces reasoning that corresponds 
to at least one oracle-correct move (Figure~\ref{fig:spatial_judge_prompt}). 
Diagnoses F1 (limited search).

\begin{figure*}[h!]
\centering
\begin{tcolorbox}[
  colback=gray!6!white,
  colframe=black,
  colbacktitle=gray!25!white,
  coltitle=gray!70!black,
  title={\textbf{Spatial Grounding Judge (SG)}},
  fonttitle=\small,
  arc=4pt,
  boxrule=0.5pt,
]
\small
Given the director's \textbf{private target view}, \textbf{current board state}, \textbf{oracle correct moves}, and \textbf{internal reasoning}, evaluate the spatial grounding quality of the director's reasoning.

\medskip
\begin{description}[leftmargin=2em, labelwidth=2em, labelsep=0.5em]
  \item[SG1.] Does the internal reasoning correctly identify at least one block missing from this director's visible wall?
  \item[SG2.] Does the reasoning avoid describing blocks or positions already correctly placed on the board?
  \item[SG3.] Does the reasoning reference the correct layer for the missing block, accounting for what is already stacked?
  \item[SG4.] Does the reasoning identify at least one action matching or closely corresponding to an oracle correct move?
  \item[SG5.] Is the implied action executable given the current board state, respecting stacking order?
  \item[SG6.] Does the reasoning correctly interpret the size of the missing block (small vs.\ large) from the target view?
  \item[SG7.] Does the reasoning stay within this director's visible wall cells rather than describing cells belonging to another director?
\end{description}

\medskip\noindent
Each question answered \textbf{Yes} / \textbf{No} / \textbf{Unclear} with a one-sentence justification. Response format: JSON with keys \texttt{SG1}--\texttt{SG7}, each containing \texttt{answer} and \texttt{reason}.
\end{tcolorbox}
\caption{Spatial Grounding (SG) judge. Evaluated once per turn across 20 structures and 20 overall turns over the collective Director messages.}
 \label{fig:spatial_judge_prompt}
\end{figure*}

\paragraph{Mind Model Judge (MM)}
Evaluates the Director's public \texttt{<message>} in the context of 
other Directors' utterances and conversation history, assessing whether 
$\tomlistener{\mathrm{joint}}$ produces a non-redundant, uniquely 
informative message the Builder can act on without clarification 
(Fig.~\ref{fig:mind_model_judge_prompt}). Diagnoses F2 (flawed pragmatics).

\begin{figure*}[h!]
\centering
\begin{tcolorbox}[
  colback=gray!6!white,
  colframe=black,
  colbacktitle=gray!25!white,
  coltitle=gray!70!black,
  title={\textbf{Mind Model Judge (MM)}},
  fonttitle=\small,
  arc=4pt,
  boxrule=0.5pt,
]
\small
Given the director's \textbf{internal reasoning}, \textbf{public message}, \textbf{other directors' messages}, and \textbf{recent conversation history}, evaluate the Theory-of-Mind quality of the director's public communication.

\medskip
\begin{description}[leftmargin=2em, labelwidth=2em, labelsep=0.5em]
  \item[MM1.] Does the public message add information not already communicated by other directors in this or the immediately preceding turn?
  \item[MM2.] Does the message avoid repeating an instruction already given and acted upon in a previous turn?
  \item[MM3.] Does the message reflect awareness of what the builder already knows from the conversation history?
  \item[MM4.] Does the message accurately translate the key finding from internal reasoning into natural language without losing critical spatial information?
  \item[MM5.] Does the message focus on information uniquely visible from this director's wall rather than information another director could equally provide?
  \item[MM6.] Is the message specific enough for the builder to execute without further clarification, naming a block type, location, and action?
  \item[MM7.] If another director gave a conflicting instruction this turn, does the message acknowledge or attempt to resolve the conflict?
  \item[MM8.] Does the message show awareness of the boundary between what this director uniquely sees and what other directors can also see?
\end{description}

\medskip\noindent
Each question answered \textbf{Yes} / \textbf{No} / \textbf{Unclear} with a one-sentence justification. Response format: JSON with keys \texttt{MM1}--\texttt{MM8}, each containing \texttt{answer} and \texttt{reason}.
\end{tcolorbox}
\caption{Mind Model (MM) Judge. Evaluated once per turn across 20 structures and 20 overall turns over the collective Director messages. Note that oracle correct moves are not provided to the MM judge during evaluation.}
 \label{fig:mind_model_judge_prompt}
\end{figure*}

\paragraph{Pragmatic Sufficiency Judge (PS)}
Evaluates the collective Director output---all three public messages 
together---against the oracle candidate set, assessing whether the 
group jointly provided sufficient information for a rational Builder 
to identify a correct move without independent spatial reasoning 
(Fig.~\ref{fig:pragmatic_judge_prompt}). Diagnoses the group-level 
communication failure defined in Sec.~\ref{sec:directors_as_bps}, 
irreducible to any individual Director's SG or MM score.

Notably the three individual failure modes (F1--F3) from~\citet{nguyen2024languagemodelsboundedpragmatic} (\Cref{sec:directors_as_bps}) are each \textbf{agent-local}: F1 (limited search) and F3 (inefficient inference) are evaluated per Director in isolation, while F2 (flawed pragmatics)---though its \textit{target} is the joint listener $\tomlistener{\mathrm{joint}}$ in \craft{}'s multi-sender setting---is still \textit{measured} per Director through the MM judge, which evaluates each Director's message individually for epistemic calibration. The fourth failure mode introduced in Sec.~\ref{sec:directors_as_bps} is categorically different: it is both measured and defined at the group level, occurring when all three Directors individually produce non-redundant, calibrated messages (F2 absent) yet their collective output still fails to identify an oracle move for the Builder. No individual Director's SG or MM score can detect this failure---it is only measurable through the PS judge, which evaluates the full collective Director output against the oracle candidate set.

\begin{figure*}[h!]
\centering
\begin{tcolorbox}[
  colback=gray!6!white,
  colframe=black,
  colbacktitle=gray!25!white,
  coltitle=gray!70!black,
  title={\textbf{Pragmatic Sufficiency Judge (PS)}},
  fonttitle=\small,
  arc=4pt,
  boxrule=0.5pt,
]
\small
Given the \textbf{current board state}, \textbf{oracle correct moves}, \textbf{director messages}, and \textbf{builder confirmation}, evaluate whether the collective director messages were pragmatically sufficient to guide the builder toward a correct move.

\medskip
\begin{description}[leftmargin=2em, labelwidth=2em, labelsep=0.5em]
  \item[PS1.] Do the director messages collectively identify at least one specific location on the board that needs a block placed or removed?
  \item[PS2.] Do the director messages collectively specify the correct block type --- both color \emph{and} size (small vs.\ large) --- for at least one oracle correct move?
  \item[PS3.] Would a rational builder reading only these messages have sufficient information to select at least one oracle correct move \emph{without} independent spatial reasoning about the target structure?
  \item[PS4.] Do the messages use precise spatial anchors that uniquely identify the target location, rather than vague relative language that could map to multiple grid positions?
  \item[PS5.] Does the builder confirmation indicate it correctly understood the directors' collective intent, regardless of whether execution succeeded?
  \item[PS6.] If the builder did not execute the correct move, was the failure primarily attributable to director underspecification rather than builder execution mechanics (wrong layer, missing span, stacking violation)? \textit{N/A if oracle move was executed.}
\end{description}

\medskip\noindent
Each question answered \textbf{Yes} / \textbf{No} / \textbf{Unclear} (or \textbf{N/A} for PS6) with a one-sentence justification. Response format: JSON with keys \texttt{PS1}--\texttt{PS6}, each containing \texttt{answer} and \texttt{reason}.
 \end{tcolorbox}
 \vspace*{-2mm}
\caption{Pragmatic Sufficiency (PS) judge. Evaluated once per turn over the collective Director messages. Six binary questions targeting location specificity, block type precision, rational sufficiency, spatial anchoring, Builder understanding, and failure attribution.}
 \label{fig:pragmatic_judge_prompt}
\end{figure*}

\clearpage

\subsection{LLM Grader Validation}
\label{app:grader-val}

\an{Due to the scale of CRAFT evaluation---thousands of turn logs across 15 Director models, 20 structures, and multiple runs---we first validated GPT-4o-mini against GPT-4o on a stratified subsample (600 director-turns for SG/MM and 295 for PS). Table~\ref{tab:judge_validation} reports exact agreement, weighted $\kappa$, and Spearman rank correlation between the two judges. Agreement is strongest on the questions central to our analysis: MM5 (unique perspective utilization) achieves $\rho=0.753$ ($p=0.001$), SG7 (wall adherence) $\rho=0.772$ ($p=0.001$), and PS3 (rational sufficiency) $\rho=0.555$ ($p=0.032$), with PS6 (failure attribution) reaching near-perfect agreement ($\kappa=0.898$). \textbf{Overall, model rankings are consistently preserved across key diagnostic dimensions, indicating that GPT-4o-mini (vs GPT4o) reliably captures the relative differences between models.} Based on this validation and for computational cost-efficiency, we use GPT-4o-mini for full-scale evaluation reported in \Cref{fig:judge_questions} and \Cref{fig:judge_permodel}, enabling tractable large-scale analysis while maintaining alignment with a stronger judge on crucial metrics in our experiments.}

\begin{table}[t]
\centering
\small
\setlength{\tabcolsep}{4pt}
\begin{tabular}{llccccc}
\toprule
\textbf{Judge} & \textbf{Q} & \textbf{Exact} & \textbf{W.$\kappa$} & \textbf{Rank $\rho$} & \textbf{$p$} \\
\midrule
\multirow{7}{*}{SG}
 & \textbf{SG1} & \textbf{0.665} & \textbf{0.265} & \textbf{0.821} & \textbf{0.000} \\
 & \textbf{SG2} & \textbf{0.547} & \textbf{0.146} & \textbf{0.697} & \textbf{0.004} \\
 & \textbf{SG3} & \textbf{0.473} & \textbf{0.260} & \textbf{0.516} & \textbf{0.049} \\
 & SG4          &         0.510  &         0.139  &         0.427  &         0.113  \\
 & \textbf{SG5} & \textbf{0.565} & \textbf{0.314} & \textbf{0.649} & \textbf{0.009} \\
 & \textbf{SG6} & \textbf{0.545} & \textbf{0.271} & \textbf{0.825} & \textbf{0.000} \\
 & \textbf{SG7} & \textbf{0.902} & \textbf{0.570} & \textbf{0.772} & \textbf{0.001} \\
\midrule
\multirow{8}{*}{MM}
 & \textbf{MM1} & \textbf{0.697} & \textbf{0.294} & \textbf{0.563} & \textbf{0.029} \\
 & MM2          &         0.565  &         0.181  &         0.226  &         0.418  \\
 & MM3          &         0.360  &         0.039  &         0.496  &         0.060  \\
 & MM4          &         0.835  &         0.456  &         0.231  &         0.407  \\
 & \textbf{MM5} & \textbf{0.723} & \textbf{0.456} & \textbf{0.753} & \textbf{0.001} \\
 & \textbf{MM6} & \textbf{0.827} & \textbf{0.568} & \textbf{0.604} & \textbf{0.017} \\
 & \textbf{MM7} & \textbf{0.810} & \textbf{0.252} & \textbf{0.682} & \textbf{0.005} \\
 & \textbf{MM8} & \textbf{0.610} & \textbf{0.274} & \textbf{0.670} & \textbf{0.006} \\
\midrule
\multirow{6}{*}{PS}
 & PS1          &         0.864  &         0.413  &         0.380  &         0.163  \\
 & PS2          &         0.736  &         0.365  &         0.402  &         0.138  \\
 & \textbf{PS3} & \textbf{0.780} & \textbf{0.413} & \textbf{0.555} & \textbf{0.032} \\
 & PS4          &         0.888  &         0.373  &         0.467  &         0.079  \\
 & PS5          &         0.736  &         0.553  &         0.052  &         0.855  \\
 & \textbf{PS6} & \textbf{0.956} & \textbf{0.898} & \textbf{0.619} & \textbf{0.014} \\
\midrule
\multirow{3}{*}{Overall}
 & SG & 0.601 & 0.281 & 0.672 & --- \\
 & MM & 0.678 & 0.403 & 0.528 & --- \\
 & PS & 0.827 & 0.503 & 0.412 & --- \\
\bottomrule
\end{tabular}
\caption{\textbf{GPT-4o-mini vs GPT-4o agreement} on a stratified subsample
selected as 2 turns per model$\times$structure for SG and MM
(15 models $\times$ 20 structures $\times$ 2 turns $= 600$ director-turns each,
yielding $4{,}200$ and $4{,}800$ question-answer pairs)
and 1 turn per model$\times$structure for PS
($n=295$ turns, $1{,}770$ pairs).
Exact agreement and linearly weighted $\kappa$ measure absolute agreement;
Rank $\rho$ reports Spearman correlation of per-model mean scores between
judges ($n=15$ models). \textbf{Bold} rows indicate questions with
significant rank preservation ($p{<}0.05$). GPT-4o applies stricter
correctness thresholds throughout, producing lower Yes rates; the
preserved model rankings confirm that relative comparisons reported
in the paper are robust to this calibration difference.}
\label{tab:judge_validation}
\end{table}

Additionally, \Cref{tab:judge_scores} shows the mean LLM grader scores using our primary grader model (GPT4o-mini) for our main results--- averaged across 3 independent runs. The low SEM values show that graders evaluate consistently. Fig.~\ref{fig:judge_permodel} shows per-model LLM grader scores across all individual judge questions, showing where individual model failures may occur. \Cref{tab:feature_correlations_judge} shows Pearson and Spearman correlations of Director communication features with task progress at turn 20, showing that a specific subset of features, such as the {\it remove gap} have strong {\it negative} correlation with task progress, even if the features themselves imply better individual communication, such as introducing novel information or leveraging the Director's unique individual perspective view.

\begin{table*}[t]
\centering
\small
\begin{tabular}{lcccc}
\toprule
\textbf{Model} & \textbf{SG $\pm$ SEM} & \textbf{MM $\pm$ SEM} & \textbf{PS $\pm$ SEM} & \textbf{Progress} \\
\midrule
\multicolumn{5}{l}{\textbf{Frontier Models}} \\
Gemini-3-Flash                & $0.669 \pm 0.010$ & $0.237 \pm 0.009$ & $0.356 \pm 0.015$ & $0.675$ \\
GPT-4o                        & $0.792 \pm 0.006$ & $0.706 \pm 0.004$ & $0.492 \pm 0.010$ & $0.588$ \\
GPT-4o-Mini                   & $0.741 \pm 0.005$ & $0.632 \pm 0.005$ & $0.457 \pm 0.006$ & $0.333$ \\
GPT-4.1-Mini                  & $0.937 \pm 0.003$ & $0.787 \pm 0.003$ & $0.481 \pm 0.005$ & $0.312$ \\
Claude-Sonnet-4.6             & $0.910 \pm 0.006$ & $0.775 \pm 0.006$ & $0.444 \pm 0.006$ & $0.285$ \\
Gemini-2.5-Flash              & $0.932 \pm 0.003$ & $0.757 \pm 0.006$ & $0.479 \pm 0.011$ & $0.257$ \\
Gemini-3.1-Flash-Lite-Preview & $0.819 \pm 0.006$ & $0.597 \pm 0.007$ & $0.481 \pm 0.008$ & $0.257$ \\
\midrule
\multicolumn{5}{l}{\textbf{Open-Weight Models}} \\
Mistral-7B                    & $0.716 \pm 0.007$ & $0.555 \pm 0.005$ & $0.509 \pm 0.009$ & $0.631$ \\
Qwen-7B                       & $0.753 \pm 0.006$ & $0.530 \pm 0.005$ & $0.514 \pm 0.012$ & $0.612$ \\
Llama-8B                      & $0.744 \pm 0.006$ & $0.491 \pm 0.006$ & $0.500 \pm 0.011$ & $0.586$ \\
Gemma-9B                      & $0.673 \pm 0.007$ & $0.507 \pm 0.005$ & $0.481 \pm 0.010$ & $0.578$ \\
Qwen-72B                      & $0.745 \pm 0.007$ & $0.503 \pm 0.008$ & $0.378 \pm 0.010$ & $0.557$ \\
Qwen-14B                      & $0.748 \pm 0.005$ & $0.523 \pm 0.004$ & $0.511 \pm 0.008$ & $0.476$ \\
DeepSeek-Lite                 & $0.142 \pm 0.005$ & $0.334 \pm 0.004$ & $0.617 \pm 0.011$ & $0.419$ \\
Qwen-32B                      & $0.742 \pm 0.006$ & $0.571 \pm 0.005$ & $0.476 \pm 0.006$ & $0.339$ \\
\midrule
\multicolumn{5}{l}{\textbf{Group Averages}} \\
Open-Weight                   & $0.658 \pm 0.069$ & $0.502 \pm 0.024$ & $0.499 \pm 0.022$ & $0.525$ \\
Frontier                      & $0.829 \pm 0.036$ & $0.642 \pm 0.067$ & $0.452 \pm 0.019$ & $0.387$ \\
\bottomrule
\end{tabular}
\caption{LLM grader scores (mean $\pm$ SEM) for spatial grounding (SG), mind modeling (MM), and pragmatic sufficiency (PS), alongside task progress. Frontier models achieve higher SG and MM scores but lower progress than open-weight models. All scores are averaged over \textbf{three independent grader runs}.}
\label{tab:judge_scores}
\end{table*}

\begin{figure*}[t]
    \centering
    \begin{subfigure}{\linewidth}
        \includegraphics[width=\linewidth]{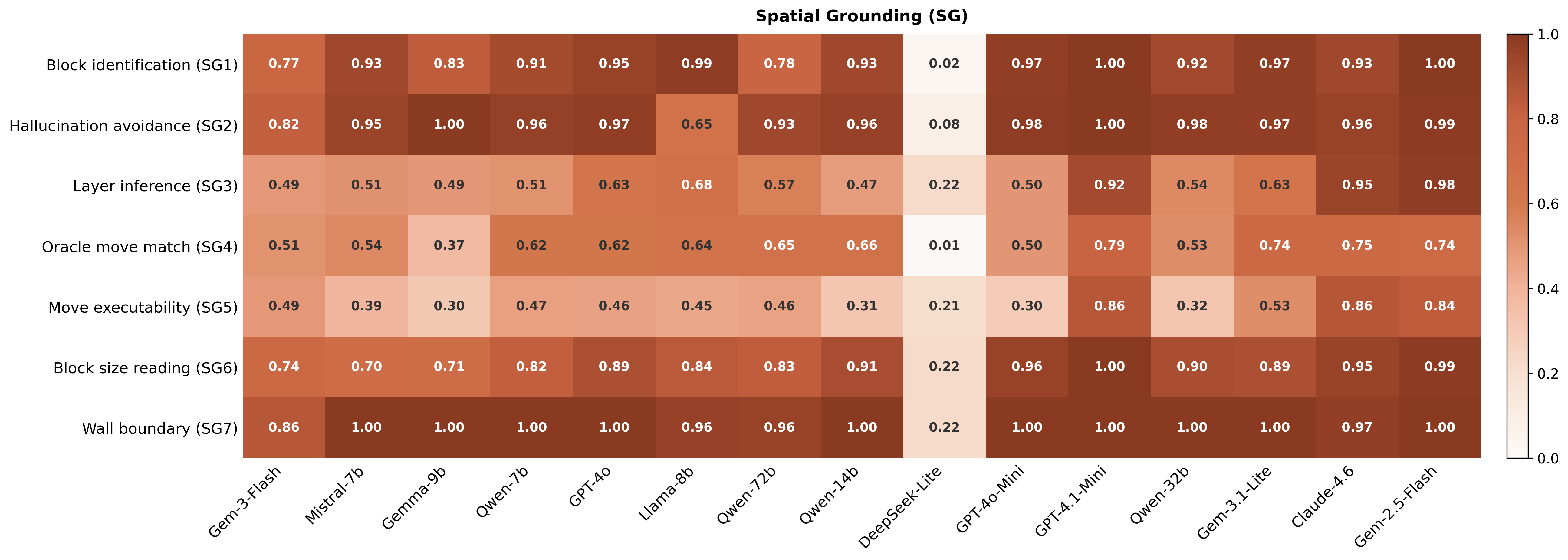}
    \end{subfigure}
    \vspace{2mm}
    \begin{subfigure}{\linewidth}
        \includegraphics[width=\linewidth]{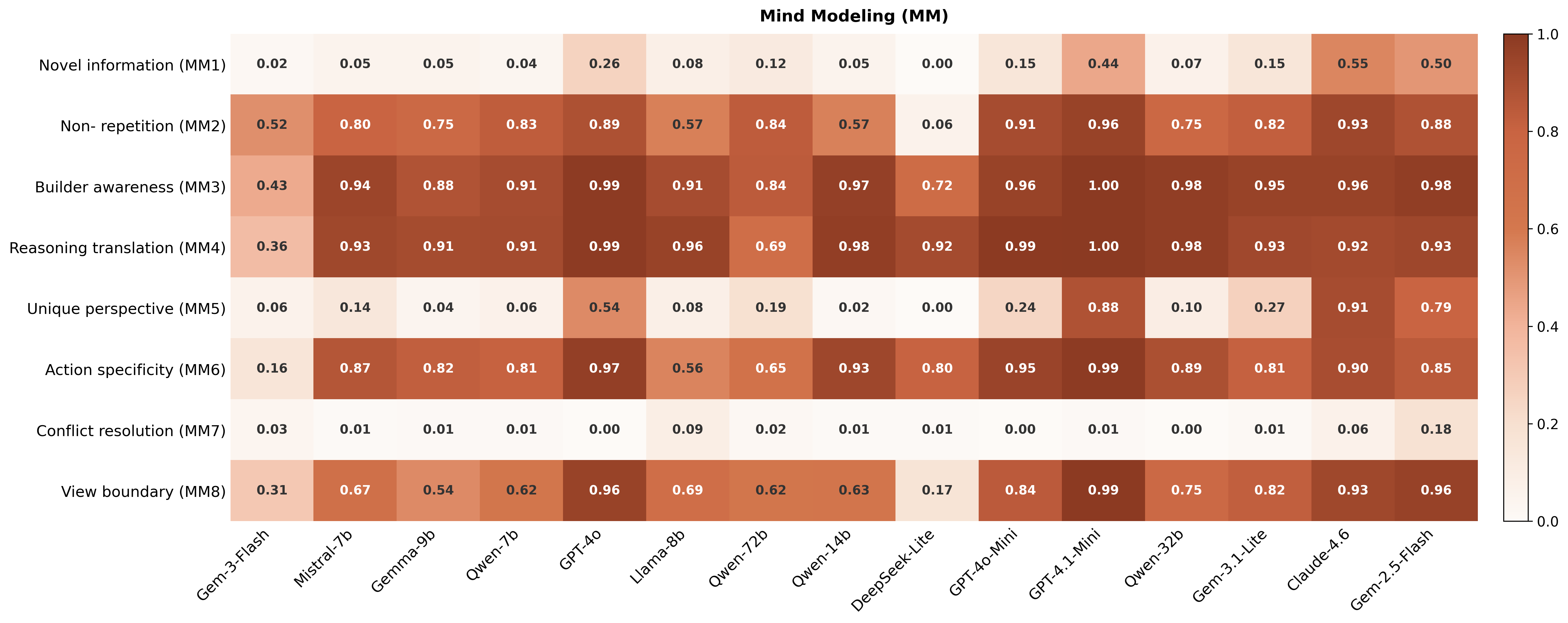}
    \end{subfigure}
    \vspace{2mm}
    \begin{subfigure}{\linewidth}
        \includegraphics[width=\linewidth]{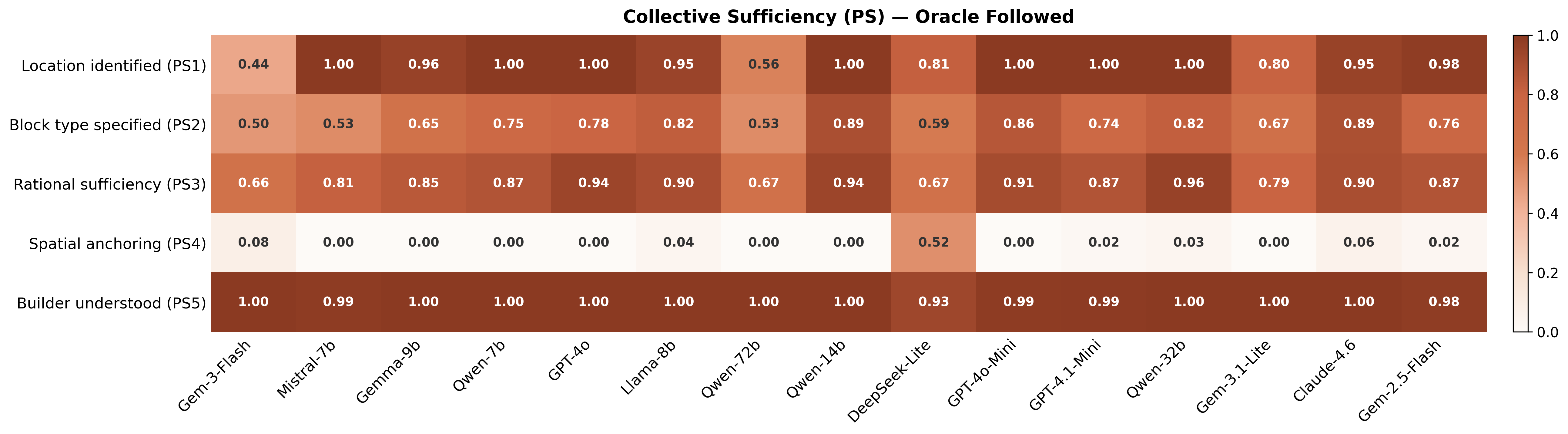}
    \end{subfigure}
    \vspace{2mm}
    \begin{subfigure}{\linewidth}
        \includegraphics[width=\linewidth]{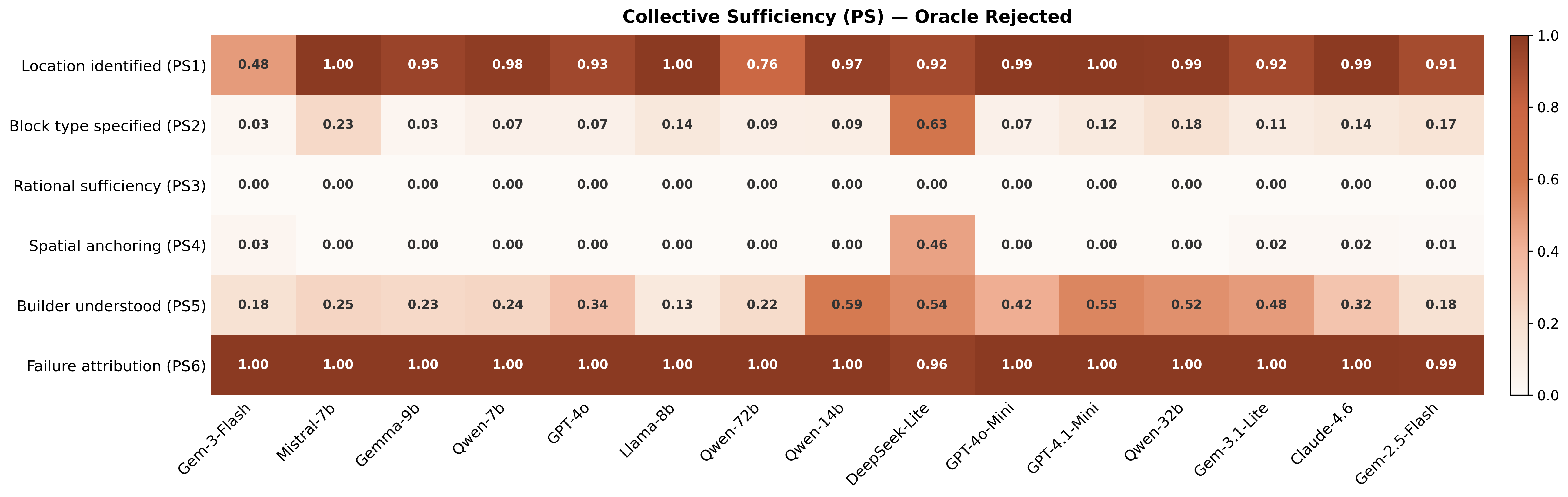}
    \end{subfigure}
    \caption{Per-model LLM grader scores across all judge questions for
             spatial grounding (top), mind modeling (second), and
             pragmatic sufficiency split by condition---oracle followed
             (third) and oracle rejected (bottom). Models are sorted by
             task progress descending along the x-axis. PS scores for
             oracle-followed turns exclude PS6 (Failure attribution),
             which is not applicable when the builder executes a correct
             move. All scores are averaged across independent grader
             runs ($n{=}3$).
             }
    \label{fig:judge_permodel}
\end{figure*}


\clearpage
\subsection{Reconciling Task Performance with LLM Judge results via “Remove-gap” analysis}
\label{app:mediation_analysis}

\begin{table}[t]
\centering
\small
\setlength{\tabcolsep}{6pt}
\begin{tabular}{lcccc}
\toprule
\textbf{Feature} & \textbf{Pearson $r$} & \textbf{$p$} & \textbf{Spearman $\rho$} & \textbf{$p$} \\
\midrule
\multicolumn{5}{l}{\textit{Behavioral}} \\
Remove gap                                    & $-0.780^{***}$ & 0.001 & $-0.785^{***}$ & 0.001 \\
\midrule
\multicolumn{5}{l}{\textit{Message calibration (MM)}} \\
Novel information (MM1)                       & $-0.619^{*}$   & 0.014 & $-0.652^{**}$  & 0.008 \\
Overall MM score                              & $-0.603^{*}$   & 0.017 & $-0.613^{*}$   & 0.015 \\
Unique perspective (MM5)                      & $-0.575^{*}$   & 0.025 & $-0.536^{*}$   & 0.039 \\
Action specificity (MM6)                      & $-0.514$       & 0.050 & $-0.311$       & 0.259 \\
Conflict resolution (MM7)                     & $-0.284$       & 0.304 & $-0.199$       & 0.478 \\
\midrule
\multicolumn{5}{l}{\textit{Spatial reasoning (SG)}} \\
Layer inference (SG3)                         & $-0.483$       & 0.068 & $-0.472$       & 0.076 \\
Move executability (SG5)                      & $-0.443$       & 0.098 & $-0.302$       & 0.274 \\
Overall SG score                              & $-0.273$       & 0.326 & $-0.554^{*}$   & 0.032 \\

\midrule
\multicolumn{5}{l}{\textit{Collective (PS)}} \\
Rational sufficiency (PS3)                    & $+0.666^{**}$  & 0.007 & $+0.633^{*}$   & 0.011 \\
\midrule
\multicolumn{5}{l}{\textit{Surface}} \\
Message length (words)                        & $+0.046$       & 0.870 & $-0.295$       & 0.286 \\
\midrule
\multicolumn{5}{l}{\textit{Mediation: unique perspective $\to$ remove gap $\to$ progress}} \\
Unique perspective alone                      & \multicolumn{4}{l}{$R^2 = 0.330$} \\
Remove gap alone                              & \multicolumn{4}{l}{$R^2 = 0.609$} \\
Unique perspective $+$ remove gap             & \multicolumn{4}{l}{$R^2 = 0.633$ \quad ($\Delta R^2 = 0.024$)} \\
Partial $r$ (unique perspective $\mid$ remove gap) & \multicolumn{4}{l}{$r = -0.247$, $p = 0.374$ (n.s.) } \\
\bottomrule
\end{tabular}
\caption{Pearson and Spearman correlations of Director communication
         features with task progress at turn 20 ($n = 15$ models).
         Features are ordered by absolute Pearson $r$.
         Judge question codes in parentheses refer to the evaluation
         framework in Sec.~\ref{app:llm_judge_prompts}.
         Significance: $^{*}p{<}0.05$, $^{**}p{<}0.01$, $^{***}p{<}0.001$.}
\label{tab:feature_correlations_judge}
\end{table}

Among the features in Table~\ref{tab:feature_correlations_judge}, MM5 and PS3 stand out as the strongest individual-level and group-level predictors of task progress respectively---and as the only features that directly implicate the correction spiral mechanism through their relationship with remove gap---motivating a mediation analysis for each.
Specifically, to test whether remove gap mediates the relationship between unique perspective utilization (MM5) and task progress, we regressed task progress on MM5 alone ($R^2 = 0.330$), remove gap alone ($R^2 = 0.609$), and both jointly ($R^2 = 0.633$). Adding MM5 to remove gap increases explained variance by only 
2.4 percentage points, and the partial correlation between MM5 and progress controlling for remove gap is non-significant ($r = -0.247$, $p = 0.374$). Remove gap fully mediates the 
MM5--progress relationship: Directors who leverage their unique perspective produce more correction-oriented instructions, 
driving over-removal that consumes the turn budget without 
advancing task progress. The direct effect of communication 
quality on task outcomes is entirely absorbed by the 
behavioral signature it produces overall.

\paragraph{Collective Pragmatic Sufficiency as an Independent Path.}

Collective pragmatic sufficiency (PS3) reveals a second, independent upstream path to the same mediator. PS3 directly predicts remove gap ($r = -0.565$, $p = 0.028$) and task progress ($r = +0.666$, $p = 0.007$; $\rho = +0.633$, $p = 0.011$), and remains a significant predictor of progress after controlling for MM score (partial $r = 0.628$, $p = 0.012$). MM5 and PS3 are not significantly correlated ($r = -0.363$, $p = 0.183$) and explain complementary variance in remove gap (MM5 alone $R^2 = 0.329$, PS3 alone $R^2 = 0.319$, jointly $R^2 = 0.475$), confirming that they operate through distinct mechanisms.

In contrast to MM5 pathway which is \textit{proactive}, the PS3 path is \textit{reactive}: when collective Director output fails to identify an oracle move, the Builder makes incorrect placements that accumulate board errors, subsequently triggering removal instructions. Both paths converge on remove gap as the proximal behavioral mechanism, consistent with the full mediation established above.

\paragraph{Correction Spirals: A Frontier Model-specific Failure Mode} Importantly, splitting by model group reveals that this mediation is \textit{asymmetric}. For \textbf{frontier} models ($n = 7$), remove gap alone explains $R^2 = 0.941$ of progress variance ($r = -0.970$, $p = 0.006$), with full mediation of MM5's effect confirmed (partial $r = 0.098$, $p = 0.876$ after controlling for remove gap). For \textbf{open-weight} models ($n = 8$), remove gap explains only $R^2 = 0.286$ of progress variance ($r = -0.535$, $p = 0.172$), and MM5 shows no significant relationship with remove gap ($r = 0.205$, $p = 0.627$). Notably, MM5 correlates \textit{positively} with progress in open-weight models ($r = +0.312$)---the opposite direction from frontier models---indicating that unique perspective utilization does not activate the correction spiral behavior when epistemic calibration is weaker overall. The correction spiral is therefore a \textbf{frontier-specific failure mode}: frontier models' superior ability to leverage unique wall perspective produces correction-oriented instructions that drive over-removal, in a way that open-weight models' weaker epistemic calibration does not.

\paragraph{Ruling Out Context Availability as a Confound in Epistemic Uniqueness (MM5).}

A potential confound in the frontier--open-weight performance asymmetry is that smaller models may receive insufficient context to reason about epistemic uniqueness. We rule this out directly: all Directors receive \textit{identical}\footnote{We do not mean identical in that every game contains the same inputs exactly, but that the information available to directors is consistent across models during game runs.} inputs at each turn---the same partial target view, current board state, and full conversation history. The information needed to determine what is uniquely visible from one's wall---i.e., what co-Directors have already communicated and what remains exclusively within one's view---is equally available to all models. Thus, the gap in MM5 scores reflects differences in \textit{reasoning over} context rather than \textit{access to} it. This implies that the correction spiral is not an artifact of context limitations or information bottlenecks, but a \textit{capability threshold effect}: sufficient epistemic calibration is required to translate available context into uniquely grounded correction instructions. Open-weight models fall below this threshold not because they see less, but because they cannot reliably act on shared context in an epistemically calibrated way. \craft{}’s physically constrained environment where over-correction directly impacts task progress, enables this threshold effect to surface clearly, rather than being masked in more permissive evaluation settings.


\clearpage

\section{Additional Experimental Results and Examples}
\label{app:additional_experimental_results}

\subsection{Connection Between Homogeneity of Director Groups and Equal Weighting in joint ToM listener (via $\lambda_i$ in Eq.~\ref{eqn:joint_tom}).} 
\label{app:lambda_consistency_validation}

As noted in \Cref{app:proofs}, we perform an empirical analysis of per-Director influence using Builder attribution or the “credit-rate” on oracle-correct turns, i.e., instances where the Builder executes a move matching the oracle, which we treat as reliable signals of task-aligned behavior (see Appendix~\ref{app:structure_and_oracle}). We compute normalized attribution weights $\hat{\lambda}_i$ based on which Director’s instruction is followed during successful execution. \textbf{As shown in Table~\ref{tab:lambda_full}, $\hat{\lambda}_i$ are approximately uniform across models, indicating that the Builder integrates Director signals symmetrically in this setting.} Importantly, these estimates reflect \textit{observed behavioral influence} rather than counterfactual causal credit. Nevertheless, their near-uniformity provides empirical support for the equal-weighting behavior implied by symmetry, while deviations in heterogeneous settings would indicate meaningful differences in agent influence and motivate learning $\lambda_i$.

 
Specifically, for each oracle-followed turn, we parse the Builder's confirmation text to identify \textit{which} Directors received attribution credit. We define the credit rate for Director $D_i$ as the fraction of oracle-followed turns in which the Builder's confirmation explicitly references $D_i$, normalized by $D_i$'s participation rate across all turns (mean $\approx 0.70$ per Director, approximately equal across all 15 models---confirming that participation frequency does not confound the analysis).

For a robust analysis, we additionally apply a degeneracy filter to remove turns where Director messages are empty or truncated, raw board state JSON dumps, or near-duplicate copies of other Directors' messages. Note that progress metrics and LLM grader runs/scores are computed over \textit{all} logged turns to capture the full distribution of Director behavior including failures. In contrast, the $\lambda_i$ credit attribution is restricted to oracle-correct turns (where the builder followed the oracle-suggested move exactly) and clean Director messages because only on these turns does Builder confirmation text reflect a meaningful attribution of instruction credit to a specific Director. Specifically, empirical $\lambda_i$ values are then computed by normalizing the clean credit rates to sum to one per model. As seen in Table~\ref{tab:lambda_full}, empirical $\lambda_i$ values are close to equal weighting across all 15 models, with mean values of $\lambda_{D1} = 0.332$, $\lambda_{D2} = 0.352$, and $\lambda_{D3} = 0.316$ overall. Deviations from equal weighting ($\lambda = 0.333$) are small and consistent: $\Delta \lambda_{D2} = +0.019$, $\Delta \lambda_{D1} = -0.001$, $\Delta \lambda_{D3} = -0.017$, with the same pattern holding for both open-weight and frontier groups. This validates the equal $\lambda_i$ assumption as a reasonable approximation for the homogeneous Director configuration used in this work.

\paragraph{Structural Information Geometry.}
The small but consistent deviation $\lambda_{D2} > \lambda_{D1} > \lambda_{D3}$ is structurally interpretable at the aggregate level. Averaged across models, Director $D_2$'s far-wall view covers $51.7\%$ of oracle-selected positions---the highest among the three---including interior cells $(1,1)$ and $(2,1)$ that are uniquely visible to $D_2$. In comparison, $D_1$ and $D_3$ cover $44.1\%$ and $36.5\%$, respectively. Despite this advantage, the Builder does not fully exploit $D_2$'s structural coverage: its average weight ($\lambda_{D2} \approx 0.352$) remains substantially below its coverage ($0.517$), while $D_3$'s weight ($\lambda_{D3} \approx 0.316$) is closest to its coverage ($0.365$, gap $-0.049$). This indicates that the Builder distributes influence more evenly than the underlying information geometry would suggest, reflecting a conservative aggregation strategy that limits dominance by any single Director despite asymmetric visibility. This is consistent with prior work~\citep{riedl2025emergent} that suggests that group “synergy” can reliably emerge in complex multi-agent systems beyond structural constraints encoded in the task setup.


\begin{table*}[t]
\centering
\small
\setlength{\tabcolsep}{3.5pt}
\begin{tabular}{lccccccccc}
\toprule
\textbf{Model} & $\lambda_{D1}$ & $\lambda_{D2}$ & $\lambda_{D3}$ & $\lambda$-std & \textbf{Filtered} & \textbf{Total} & \textbf{Oracle} & \textbf{Orc\%} & \textbf{Deg\%} \\
\midrule
\multicolumn{10}{l}{\textbf{Frontier Models}} \\
\midrule
Claude-Sonnet-4.6             & 0.336 & 0.363 & 0.302 & 0.0250 & 35  & 400 & 176 & 44.0\% & 11.8\% \\
Gemini-2.5-Flash              & 0.340 & 0.351 & 0.308 & 0.0182 & 48  & 400 & 164 & 41.0\% & 15.7\% \\
Gemini-3-Flash                & 0.344 & 0.360 & 0.296 & 0.0270 & 84  & 385 & 269 & 69.9\% & 21.3\% \\
Gemini-3.1-Flash-Lite-Preview & 0.323 & 0.347 & 0.331 & 0.0100 & 28  & 400 & 158 & 39.5\% & 10.1\% \\
GPT-4.1-Mini                  & 0.334 & 0.320 & 0.346 & 0.0106 & 26  & 400 & 171 & 42.8\% & 8.9\% \\
GPT-4o-Mini                   & 0.342 & 0.341 & 0.317 & 0.0118 & 39  & 400 & 192 & 48.0\% & 11.6\% \\
GPT-4o                        & 0.317 & 0.376 & 0.307 & 0.0306 & 47  & 397 & 262 & 66.0\% & 10.5\% \\
\midrule
\multicolumn{10}{l}{\textbf{Open-weight Models}} \\
\midrule
DeepSeek-Lite                 & 0.334 & 0.354 & 0.311 & 0.0174 & 150 & 400 & 248 & 62.0\% & 33.6\% \\
Gemma-9b                      & 0.325 & 0.342 & 0.333 & 0.0067 & 52  & 400 & 245 & 61.3\% & 12.9\% \\
Llama-8b                      & 0.316 & 0.368 & 0.317 & 0.0244 & 89  & 396 & 260 & 65.7\% & 20.1\% \\
Mistral-7b                    & 0.315 & 0.373 & 0.312 & 0.0281 & 62  & 388 & 258 & 66.5\% & 13.2\% \\
Qwen-14b                      & 0.340 & 0.360 & 0.300 & 0.0253 & 51  & 391 & 218 & 55.8\% & 12.1\% \\
Qwen-32b                      & 0.327 & 0.336 & 0.338 & 0.0049 & 30  & 400 & 167 & 41.8\% & 10.5\% \\
Qwen-72b                      & 0.325 & 0.361 & 0.314 & 0.0200 & 61  & 399 & 244 & 61.2\% & 15.2\% \\
Qwen-7b                       & 0.318 & 0.352 & 0.331 & 0.0141 & 66  & 390 & 268 & 68.7\% & 13.0\% \\
\bottomrule
\end{tabular}
 
\caption{Empirical $\lambda_i$ weights and associated diagnostics. $\lambda_{D1}$, $\lambda_{D2}$, $\lambda_{D3}$ denote the empirical credit share attributed to each Director by the Builder on oracle-correct turns, normalized to sum to one. \textbf{Oracle} is the number of turns where the Builder executed a verified oracle-correct move. \textbf{Orc\%} is the fraction of total turns corresponding to oracle-correct moves. \textbf{Deg\%} denotes the percentage of degenerate credits removed during filtering. Total represents the total number of actual turns evaluated (termination condition) across 20 test structures.}
\label{tab:lambda_full}
\end{table*}

\paragraph{Remove-Gap Evolution Across Turns} Figs.~\ref{fig:remove_evolution_base} and \ref{fig:remove_evolution_frontier} show evolution of Builder 
remove actions across turns against oracle-prescribed remove actions across all evaluations for all models. This comparison highlights the {\it over-removal} that proprietary frontier models typically engage in when compared to open-weight models. Even larger open-weight models engage in over-removal, leading to {\it correction spirals} (Fig.~\ref{fig:qwen32b_case_detailed}, \Cref{tab:qwen32b_trace}).



\begin{figure}[h!]
  \centering
  \includegraphics[width=\linewidth]{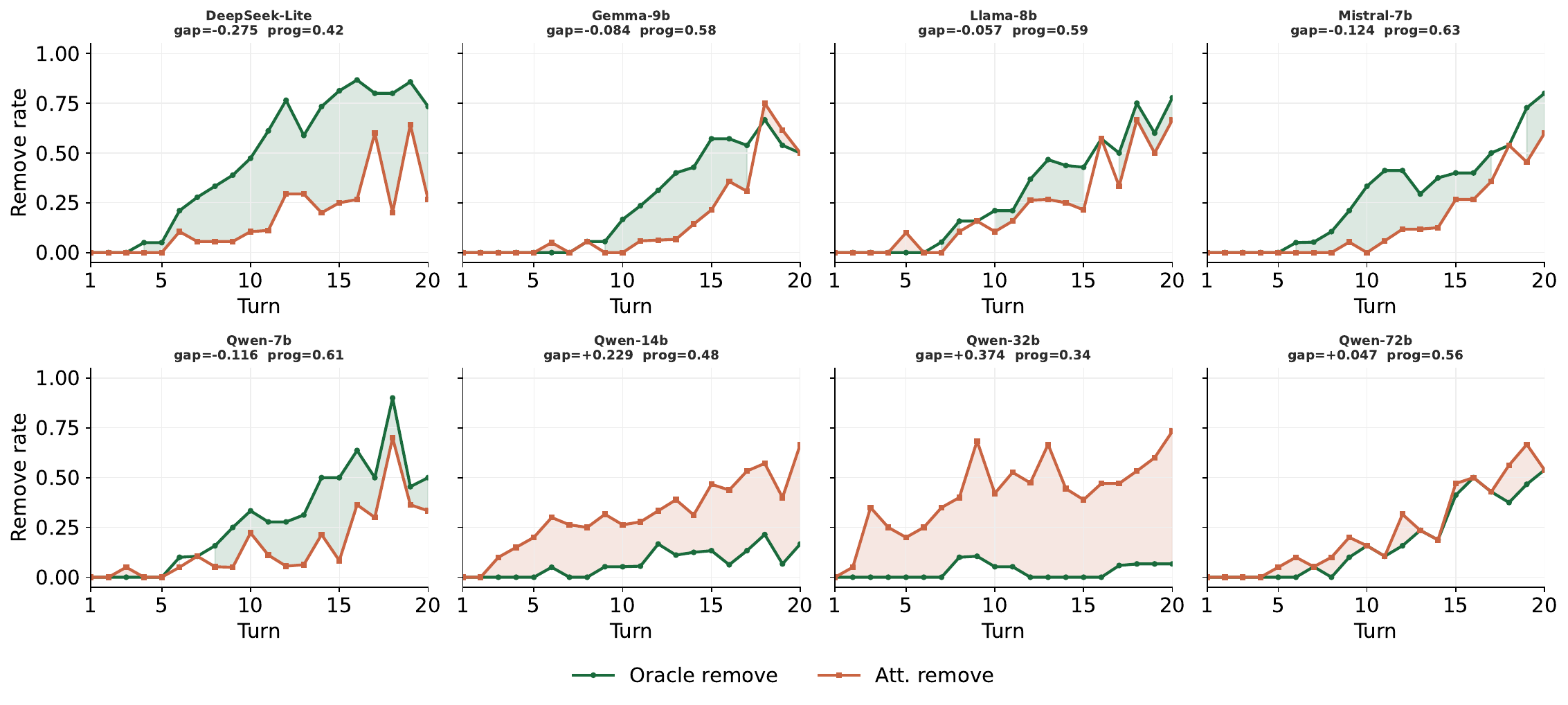}
  \caption{Oracle-prescribed vs.\ attempted remove rate per turn for all 
           base open-weight models.}
  \label{fig:remove_evolution_base}
\end{figure}

\begin{figure}[h!]
  \centering
  \includegraphics[width=\linewidth]{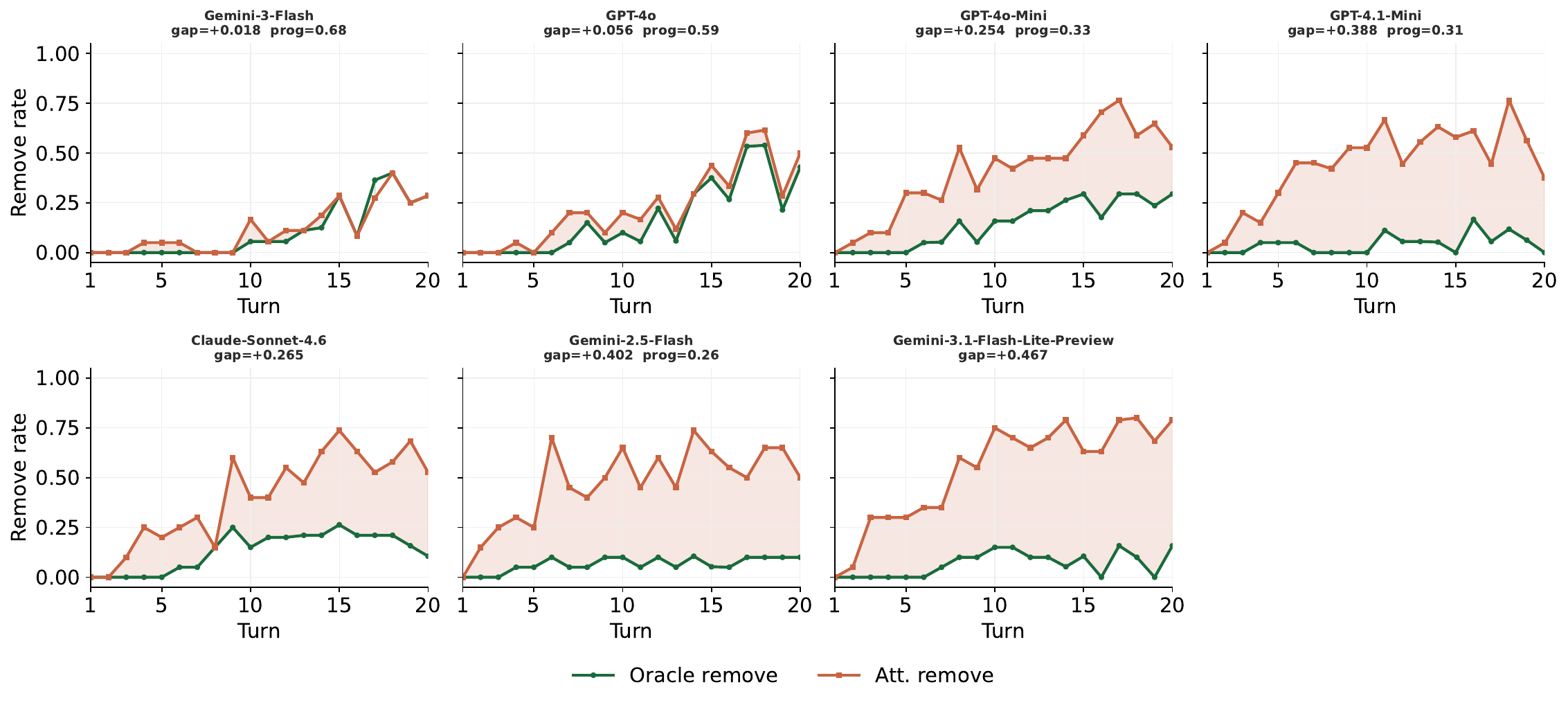}
  \caption{Oracle-prescribed vs.\ attempted remove rate per turn for all 
           frontier and proprietary models.  }
  \label{fig:remove_evolution_frontier}
\end{figure}

\paragraph{Failed Moves}
During our evaluation runs on 20 target structures, we found 47 turns with no recorded move attempt (30 for base models, 17 for proprietary models). Inspection of the raw logs revealed these are attributable to malformed or
empty Director responses that left the Builder with insufficient information
to construct a valid action, rather than any systematic model failure.
Qwen-72b accounts for the majority of base model cases (23 of 30). These turns are excluded for the involved models in the \texttt{FAIL} or failed move rates in \Cref{tab:oracle_performance}. Additionally, the Builder agent's clarification rate is negligible across all models---at most one
\texttt{CLARIFY} response per model across 385--400 turns (rate $\leq 0.003$)---confirming
that Builders consistently attempt moves rather than request additional information,
even when Director instructions could be ambiguous.


\begin{table}[h!]
\centering
\small
\setlength{\tabcolsep}{6pt}
\begin{tabular}{lccc}
\toprule
\textbf{Model} & \textbf{Orc\_Remove$\downarrow$} & \textbf{Att\_Remove} & \textbf{Gap} \\
\midrule
\textit{Open-weight models} \\
\underline{DeepSeek-Lite} & 0.433 & 0.148 & \underline{$-$0.275} \\
Mistral-7b    & 0.238 & \textbf{0.124} & \textbf{$-$0.124} \\
Qwen-7b       & 0.238 & 0.205 & $-$0.116 \\
Gemma-9b      & 0.211 & 0.122 & $-$0.084 \\
Llama-8b      & 0.232 & 0.277 & $-$0.057 \\
Qwen-72b      & 0.162 & 0.245 & $+$0.047 \\
Qwen-14b      & 0.062 & 0.355 & $+$0.229 \\
Qwen-32b      & \textbf{0.027} & 0.448 & $+$0.374 \\
\midrule
\textit{Frontier models} \\
Gemini-3-Flash        & \textbf{0.073} & \textbf{0.196} & \textbf{$+$0.018} \\
GPT-4o                & 0.146 & 0.280 & $+$0.056 \\
Claude-Son.-4.6       & 0.130 & 0.395 & $+$0.265 \\
GPT-4o-Mini           & 0.139 & 0.432 & $+$0.254 \\
GPT-4.1-Mini          & 0.040 & 0.462 & $+$0.388 \\
Gemini-2.5-Flash      & 0.065 & 0.468 & \underline{$+$0.402} \\
Gemini-3.1-Flash-Lite & 0.066 & 0.540 & $+$0.467 \\
\midrule
Overall               & 0.148 & 0.314 & $+$0.135 \\
\midrule
\multicolumn{4}{l}{\textit{Turn-level Spearman $\rho$ vs oracle adherence}} \\
\multicolumn{1}{l}{Gap}         & \multicolumn{3}{l}{$\rho = -0.543,\ p{<}0.001$} \\
\multicolumn{1}{l}{Att\_Remove} & \multicolumn{3}{l}{$\rho = -0.569,\ p{<}0.001$} \\
\bottomrule
\end{tabular}
\caption{Remove action rates per model. Oracle remove rate reflects 
what the \craft{} game engine considers necessary given the current 
board state; attempted remove rate reflects what Directors instructed 
the Builder to do. Negative gap indicates under-removal (Builder 
removes less than oracle prescribes); positive gap indicates 
over-removal. Bold = best within group. Underline = largest absolute 
gap. Turn-level Spearman correlations use full oracle adherence as 
outcome ($n=5,382$).}
\label{tab:remove_rates}
\end{table}

\begin{figure*}[h!]
    \centering
    \includegraphics[width=\linewidth]{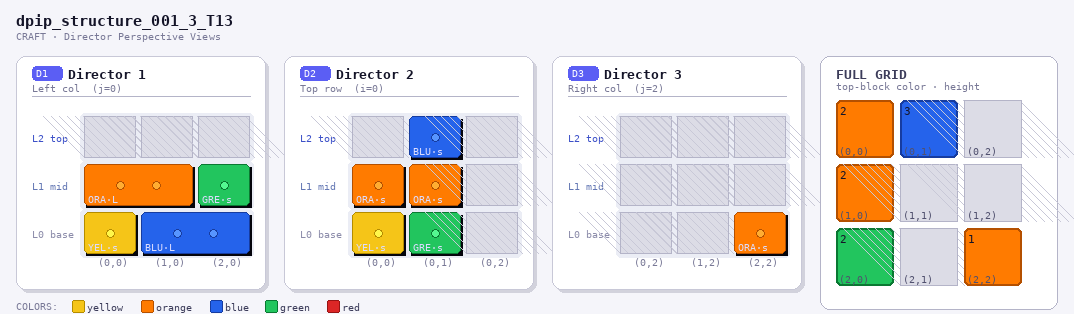}
    \caption{\textbf{Three turns, zero progress: a \craft{} correction spiral (Qwen-32b, \texttt{structure\_001}, T10--T14)} (see \Cref{tab:qwen32b_trace} 
for the full turn-by-turn trace). D1 (left wall) sees a large 
\borange{orange} domino at layer~1 spanning \texttt{(0,0)--(1,0)} 
sitting directly on top of a small \byellow{yellow} at \texttt{(0,0)} 
layer~0 and a large \bblue{blue} domino spanning \texttt{(1,0)--(2,0)} 
at layer~0. The \borange{orange} domino is therefore the \emph{top} 
block at both positions and is directly accessible for removal --- 
yet D1 instructs ``remove the large orange block from the 
\emph{bottom} layer,'' targeting layer~0 which is occupied by the 
\bblue{blue} domino. D2 (far wall) independently sees \borange{orange} 
small blocks at layer~1 on \texttt{(0,0)} and \texttt{(0,1)} but 
similarly instructs removal from ``bottom left corner,'' again 
specifying layer~0. D3 (right wall) sees only a single \borange{orange} 
block at \texttt{(2,2)} layer~0 and has no visibility into either 
problem position, remaining entirely unaware of the “deadlock”. 
The full grid minimap shows five occupied positions with maximum 
height~3 at \texttt{(0,1)}; the \borange{orange} blocks at 
\texttt{(0,0)} and \texttt{(1,0)} each sit at height~2 as the 
topmost block, confirming they are accessible. Despite this, neither 
D1 nor D2 corrects the layer specification across turns 10, 13, and 
14---the Builder's oracle recommends forward placement moves throughout 
(\okbox{\texttt{place gs @(0,0) L2}}, 
\okbox{\texttt{place rl @(1,0) L2}}) but the \textbf{Builder follows the 
Director consensus} and attempts the \textit{same} failed remove three times, 
consuming turns without any board state change.}
\label{fig:qwen32b_case_detailed}
    
\end{figure*}

\begin{table}[h!]
\centering
\small
\setlength{\tabcolsep}{4pt}
\begin{threeparttable}
\begin{tabular}{lp{3.5cm}p{7.5cm}}
\toprule
\textbf{Turn} & \textbf{Oracle (top move)} & \textbf{Director messages \& outcome} \\
\midrule

T10 & \okbox{\texttt{place gs @(0,0) L2}} 
    & D1: ``remove the large \borange{orange} block from the 
      middle-left of my \failbox{bottom layer}''
      $\cdot$ D2: ``remove \borange{orange} from bottom left corner''
      $\cdot$ \textit{Builder follows D1} $\rightarrow$ 
      \failbox{\texttt{Cannot remove layer 0 at (1,0)}} \\

T11 & \okbox{\texttt{place gs @(0,0) L2}}
    & D2: same remove instruction $\cdot$ D3: ``confirm small 
      orange in bottom right''
      $\cdot$ \textit{Builder follows D3} $\rightarrow$ 
      \okbox{\texttt{remove os @(0,2) L0}} $\checkmark$ 
      progress=0.355 \\

T12 & \okbox{\texttt{place gs @(0,0) L2}}
    & D1: same remove instruction $\cdot$ D3: ``add small orange 
      to bottom right''
      $\cdot$ \textit{Builder follows D3} $\rightarrow$ 
      \okbox{\texttt{place os @(2,2) L0}} $\checkmark$ 
      progress=0.386 \\

\midrule
T13 & \okbox{\texttt{place gs @(0,0) L2}}
    & D1: ``get rid of large \borange{orange} from 
      \failbox{bottom layer} middle-left''
      $\cdot$ D2: ``remove \borange{orange} from 
      \failbox{bottom left corner}''
      $\cdot$ D3: unaware of deadlock, confirms own wall
      $\cdot$ \textit{Builder follows D1+D2 consensus} $\rightarrow$
      \failbox{\texttt{Cannot remove layer 0 at (1,0)}} \\

T14 & \okbox{\texttt{place gs @(0,0) L2}}
    & D1: identical instruction as T13
      $\cdot$ D3: ``\textit{focus on removing the large orange 
      from D1's bottom layer}'' (endorses wrong instruction)
      $\cdot$ \textit{Builder follows D1+D3 consensus} $\rightarrow$
      \failbox{\texttt{Cannot remove layer 0 at (1,0)}} \\

\bottomrule
\end{tabular}
\begin{tablenotes}
\small
\item Board state at \texttt{(1,0)} is \texttt{[bl, ol]} throughout 
T10--T14: \bblue{bl} at layer~0, \borange{ol} at layer~1. 
The \borange{orange} block is accessible at layer~1 but directors 
consistently specify \failbox{layer~0}. Oracle recommends placement 
moves throughout; no director ever issues the correct prerequisite 
\okbox{\texttt{remove ol @(1,0) L1}}.
\end{tablenotes}
\end{threeparttable}
\caption{Turn-by-turn trace of the Qwen-32b correction spiral on 
\texttt{structure\_001}. Despite the oracle recommending forward 
placement moves at every turn, D1 and D2 repeatedly issue remove 
instructions targeting the wrong layer, while D3 remains unaware 
of the deadlock. The Builder escapes the spiral on T11 and T12 by 
following D3 instead, but returns to the failed pattern when D1 
and D2 reach consensus on T13--14.}
\label{tab:qwen32b_trace}
\end{table}

\clearpage
\section{Experimental Configuration}
\label{app:exp-config}

\paragraph{Hardware Setup and Decoding Settings}
All open-weight Director models were served locally on a machine with two NVIDIA RTX PRO 6000 Blackwell GPUs (96 GB VRAM each), using {\tt bfloat16} precision with 4-bit quantization for Qwen-32B and Qwen-72B. A single 20-turn game on one structure took approximately 5 minutes for smaller open-weight models (e.g., Qwen-7B, Llama-8B) and up to 15 minutes for the 72B model. Frontier API models were called sequentially with per-model rate-limiting; a single GPT-4o or GPT-4o-Mini run (20 turns) over one structure took approximately 12 minutes owing to API latency across 20 turns. All runs used a fixed random seed per structure-run index pair for reproducibility, with each model evaluated with LLM graders across 3 independent runs and 20 structures. Open-weight models use greedy decoding (temperature=0) for reproducibility, 
since local inference is fully deterministic; frontier API models use 
temperature=0.7 as API providers do not guarantee deterministic outputs 
even at temperature=0~\citep{gpt4}.\footnote{\url{https://152334h.github.io/blog/non-determinism-in-gpt-4/}}. 

\paragraph{Termination condition} In \craft{}, games are terminated early upon full structure completion to save additional compute needed for redundant turns; across all 15 models, 16 of 300 total game runs (5.3\%) reached full completion before turn 20, with Gemini-3-Flash and Mistral-7B achieving the most early completions (4 each), consistent with their leading task progress scores. Where the conversation history accumulates $>$ 50 total messages during a game, we truncate it to 40 previous messages for the remaining turns to reduce context-exposure~\citep{hongcontext} and cost. 


\clearpage

\end{document}